\documentclass[letterpaper]{article} % DO NOT CHANGE THIS
\usepackage[preprint]{aaai2027}  % DO NOT CHANGE THIS
\usepackage[hyphens]{url}  % DO NOT CHANGE THIS
\usepackage{graphicx} % DO NOT CHANGE THIS
\usepackage{natbib}  % DO NOT CHANGE THIS AND DO NOT ADD ANY OPTIONS TO IT
\usepackage{caption} % DO NOT CHANGE THIS AND DO NOT ADD ANY OPTIONS TO IT

\usepackage{algorithm}
\usepackage{algorithmic}
\usepackage{amsmath}

\usepackage{subcaption}
\usepackage{bm}
\usepackage{amssymb,amsthm}
\newtheorem{theorem}{Theorem}
\newcommand{\Acc}{\mathrm{Acc}}
\usepackage{multirow}
\usepackage{tabularx}

\usepackage{newfloat}
\usepackage{listings}
\DeclareCaptionStyle{ruled}{labelfont=normalfont,labelsep=colon,strut=off} % DO NOT CHANGE THIS
\floatstyle{ruled}
\newfloat{listing}{tb}{lst}{}
\floatname{listing}{Listing}

\usepackage{booktabs}

\title{Learning Early-to-Final Solution Consistency for MILP Acceleration}
\author {
    Guanlin Li\textsuperscript{\rm 1, \rm 2}\equalcontrib,
    Chengrui Gao\textsuperscript{\rm 1, \rm 2}\equalcontrib,
    Chenguang Wang\textsuperscript{\rm 1, \rm 2},
    Haopu Shang\textsuperscript{\rm 1, \rm 2},
    Zherong Zhang\textsuperscript{\rm 1, \rm 2},
    Ke Xue\textsuperscript{\rm 1, \rm 2},
    Jixiang Lu\textsuperscript{\rm 3},
    Weiyong Yang\textsuperscript{\rm 3},
    Chao Qian\textsuperscript{\rm 1, \rm 2}\thanks{Corresponding author. Email: \protect\url{qianc@lamda.nju.edu.cn}.\\
    }
}
\affiliations {
    \textsuperscript{\rm 1}State Key Laboratory of Novel SoftwareTechnology, Nanjing University\\
    \textsuperscript{\rm 2}School of Artificial Intelligence,Nanjing University\\
    \textsuperscript{\rm 3}State Key Laboratory of Technology and Equipment for Defense Against Power System Operational Risks, Nari Technology Co., Ltd.
}
\newcommand{\ours}{EnCore}
\newcommand{\nd}{Neural Diving}
\newcommand{\ps}{Predict-and-Search}
\newcommand{\apollo}{Apollo-MILP}

\begin{document}

\maketitle

\begin{abstract}
Mixed-Integer Linear Programming (MILP) is a fundamental problem class in operations research and combinatorial optimization, with broad applications to industrial decision-making. Owing to their NP-hardness, however, modern solvers may struggle to find high-quality solutions for challenging MILP instances within practical time limits. Recent learning-based approaches seek to accelerate MILP solving by directly predicting high-quality solutions from static instance-level features, such as variable-constraint bipartite graphs. Yet accurate solution prediction from instance features alone is difficult, and these methods largely overlook the information revealed during the solver's search process. In this paper, we find that solutions produced at the early search stage of MILP solvers, which are computationally cheap to obtain, are often structurally close to the solutions found after full-budget search. 
Motivated by this observation, we propose a new solver-informed paradigm that shifts the learning target from variable assignment to \textit{early-to-final consistency}: for each variable, we predict whether its early-stage assignment should persist in full-budget solutions. The predicted consistency naturally guides downstream search, for instance by fixing the assignments deemed consistent. At inference time, we further ensemble consistency predictions across multiple early-stage solutions to improve robustness.
Experiments across four MILP benchmarks show our method improves prediction-guided search across diverse downstream pipelines. With Gurobi, our proposed method reduces the primal gap by 56.9\% on average and closes it completely on combinatorial auction instances. Besides, we transferred the Gurobi-trained model zero-shot to SCIP without adaptation, achieving a 36.4\% average gap reduction across benchmarks. Code is available at: \protect\url{https://github.com/lamda-bbo/EnCore}.
\end{abstract}

% Uncomment the following to link to your code, datasets, an extended version or similar.
% You must keep this block between (not within) the abstract and the main body of the paper.
% Make sure that you do not de-anonymize yourself with these links.
% \begin{links}
%     \link{Code}{https://aaai.org/example/code}
%     \link{Datasets}{https://aaai.org/example/datasets}
%     \link{Extended version}{https://aaai.org/example/extended-version}
% \end{links}

\section{Introduction}

Mixed-Integer Linear Programming (MILP) is a fundamental modeling language for combinatorial optimization and decision making, with broad applications in scheduling~\cite{floudas2005mixed}, routing~\cite{Laporte1992345}, network design~\cite{CRAINIC2000272} and production planning~\cite{NAM1992255}. Prevailing MILP solvers employ branch-and-bound~\cite{land1960automatic} frameworks to systematically explore the combinatorial search space while guaranteeing optimality. Built upon this, the MILP community has developed a variety of sophisticated techniques~\cite{padberg1991branch,achterberg2009scip} over the past decades, including presolve, cutting planes, primal heuristics, and branching strategies, which further reduce the search space and accelerate the solving process. However, solving MILP remains computationally expensive for modern solvers due to their inherent NP-hardness~\cite{karp-reducibility}. 

In real-world applications, most MILP instances are derived from similar problem classes with varying parameters, exhibiting shared structural patterns that can be naturally exploited by machine learning methods. As investigated by \citet{bengio2021machine}, learning models can capture these patterns from instance characteristics and historical solving experiences, and infer effective optimization strategies. Such learned knowledge can be integrated into MILP solvers to accelerate the optimization process~\cite{gasse2019exact,li2023configure,nair2020solving}. Existing learning-based methods for MILP can be broadly classified into three types: (1) \textit{learning solver decisions}, where machine learning models are employed to guide individual solver component, such as branching, node selection, and cutting plane selection~\cite{khalil2016learning,gasse2019exact,gupta2020hybrid,ferber2020mipaal,tang2020learning,ling2024stop,Strang2026planning}; (2) \textit{learning solver configurations}, where models automatically tune algorithmic parameters or select suitable configurations for different instances~\cite{hutter2009paramils,li2023configure}; and (3) \textit{learning solution predictions}, where models directly predict complete or partial solutions to provide high-quality initial solutions or reduce the search space~\cite{nair2020solving}.

Among the solution prediction methods, recent advances have focused on two directions: improving learning-guided search mechanisms and enhancing prediction capacity. On the search side, early studies such as neural diving~\cite{nair2020solving} fixed high-confidence variable assignments predicted by neural models, producing reduced MILP instances for efficient downstream optimization. Building upon this idea, \citet{han2023gnn} replaced hard variable fixing with a trust-region search strategy, reducing the risk of being trapped by inaccurate assignment. More recently, \citet{liu2025apollomilp} proposed to alternate prediction with trust-region correction, further improving robustness against prediction errors.  
On the prediction side, \citet{wang2026milpnet} represented MILPs as geometric feature sequences and used multi-scale attention to overcome expressiveness limitations of graph-based models. Meanwhile, \citet{li2026fmip} introduced a flow-matching generator that jointly models integer and continuous variables to enhance solution generation quality, while \citet{pu2026cocomilp} exploited variable interactions through inter-variable contrastive learning and an intra-constraint competitive graph network. However, these methods still follow the conventional paradigm of predicting a complete solution directly from the MILP formulation. Under this paradigm, the model is required to determine all variable assignments, leading to a prediction task that approaches the complexity of solving the original MILP.

In this paper, we begin by analyzing how primal solutions evolve during the search of modern MILP solvers. Across all evaluated benchmarks, we observe that the solution quality improves rapidly in the early stage of the search and then progresses much more slowly with only local corrections, indicating that early incumbents are highly informative about the final, full-budget assignments. On workload apportionment, for instance, an early incumbent already agrees with the final solution on \(95.63\%\) of the binary variables (as shown in Figure~\ref{fig:early_solution_quality}). Given such an informative starting point, the optimal solution need not be reconstructed from scratch. Conditioned on what the solver has already produced, the learning task reduces to deciding which of its assignments to trust.

Based on this insight, we propose a new \emph{solver-informed} paradigm that reformulates solution prediction as an \emph{early-to-final consistency estimation} problem. Given an MILP graph \(G\), an early solution \(X^{ES}\), and a reference solution \(X^*\) obtained from a full-budget run, conventional solution prediction learns \(P(X^*\mid G)\) and thus requires the model to infer all variable assignments from static instance features alone. In contrast, we estimate
\(
P(x_i^*=x_i^{ES}\mid G,X^{ES}),
\)
namely the probability that each early assignment persists in the final solution. 
This target can substantially ease the learning task. Conditioned on the rich anchor \(X^{ES}\), the model only needs to identify a relatively small fraction of assignments that have yet to stabilize.
Theoretically, we prove that conditioning on the early solution can strictly increases the best achievable accuracy of predicting the full-budget solution, under a mild posterior-crossing condition, and that this information gain provably persists under finite-sample model selection. 
Moreover, since the predicted consistent assignments are drawn from a feasible early solution, fixing them preserves feasibility and concentrates the subsequent search on assignments that the solver has left unresolved.

We evaluate our method on multiple MILP benchmarks with Gurobi~\cite{gurobi} as the target solver and integrate it into existing prediction-based search frameworks. Experimental results show that our proposed refiner consistently reduces the primal gap in most evaluated settings. Specifically, it reduces the primal gap of Predict-and-Search~\cite{han2023gnn} by 56.9\% on average and closes the gap on combinatorial auction instances. Furthermore, without retraining, the model trained with Gurobi-generated data can be directly transferred to SCIP~\cite{achterberg2009scip}. These results demonstrate that the proposed early-to-final consistency prediction paradigm is clearly effective and showcases transferability across different solvers.
\section{Preliminaries}

\subsection{Mixed Integer Linear Programming}

A mixed-integer linear program (MILP) is an optimization problem with a linear objective and linear constraints, where a subset of decision variables is restricted to integer values. We formulate a MILP instance as
\begin{equation}
\begin{aligned}
\min_{x \in \mathbf{R}^{n}} \quad & \bm{c}^\top \bm{x} \\
\text{s.t.} \quad & \mathbf{A} \bm{x} \le \bm{b}, \\
& \bm{l} \le \bm{x} \le \bm{u}, \\
& x_i \in \mathbf{Z} \; \forall i \in \mathcal{I},
\end{aligned}
\label{eq:milp}
\end{equation}
where \(\bm{x}=(x_1,\ldots,x_n)\) is the vector of decision variables, \(\bm{c} \in \mathbf{R}^n\) is the objective coefficient vector, \(\mathbf{A} \in \mathbf{R}^{m \times n}\) is the constraint matrix, and \(\bm{b} \in \mathbf{R}^m\) is the right-hand-side vector. The bounds \(\bm{l},\bm{u} \in (\mathbf{R} \cup \{\pm\infty\})^n\) defines the variable domains, and \(\mathcal{I} \subseteq \{1,\ldots,n\}\) is the index set of variables required to take integer values. An important special case is that of binary variables, for which \(x_i \in \{0,1\}\). General integer variables can be transformed into binary variables using standard preprocessing techniques.

\subsection{Graph Representations for MILP}

Learning-based MILP methods commonly encode an instance as a bipartite graph~\cite{Salvagnin2016Detecting,gasse2019exact}. For the MILP formulated in Eq.~(\ref{eq:milp}), the associated graph is defined as \(G=(\mathcal{V},\mathcal{C},\mathcal{E})\), where \(\mathcal{V}=\{v_1,\ldots,v_n\}\) and \(\mathcal{C}=\{q_1,\ldots,q_m\}\) denote the sets of variable nodes and constraint nodes, corresponding to decision variables and constraints, respectively. An edge \((v_i,q_j) \in \mathcal{E}\) is introduced when the coefficient \(\mathbf{A}_{ji}\) is non-zero.

Each node is initialized with features describing its role in the MILP~\cite{han2023gnn}. A variable node \(v_i\) contains the objective coefficient \(c_i\), the lower and upper bounds \(l_i\) and \(u_i\), and an indicator of whether \(x_i\) is continuous or integer. A constraint node \(q_j\) contains the right-hand-side value \(b_j\) and an encoding of its constraint sense, such as \(\le\), \(=\), or \(\ge\). In addition, each edge carries the matrix coefficient \(A_{ji}\) as its edge feature. Collectively, these node and edge attributes provide a description of the local algebraic structure of the MILP instance, which serves as the input representation for graph neural networks~(GNN).

\subsection{Solution Predictions for MILP}

Solution prediction methods learn a mapping from an MILP instance to variable assignments, which can be exploited to reduce the search space for acceleration. Most existing predictors take the static bipartite graph \(G\) as input and employ a graph convolutional network (GCN) or a related message-passing GNN to compute variable embeddings, from which the final predictions are decoded. 

Let \(h_i^{(\ell)}\) be the embedding of variable node \(v_i\) at layer \(\ell\), and let \(g_j^{(\ell)}\) be the embedding of constraint node \(q_j\). A common bipartite message-passing layer first aggregates messages from variables to constraints and then sends messages back to variables:
\begin{equation*}
\begin{array}{rcl}
    g_j^{(\ell+1)}
    &=&
    \phi_C^{(\ell)}
    \left(
        g_j^{(\ell)},
        \sum_{i:(v_i,q_j)\in\mathcal{E}}
        \psi_C^{(\ell)}(h_i^{(\ell)}, A_{ji})
    \right), \\[2mm]
    h_i^{(\ell+1)}
    &=&
    \phi_V^{(\ell)}
    \left(
        h_i^{(\ell)},
        \sum_{j:(v_i,q_j)\in\mathcal{E}}
        \psi_V^{(\ell)}(g_j^{(\ell+1)}, A_{ji})
    \right),
\end{array}
\end{equation*}
where \(\psi_C^{(\ell)}\) and \(\psi_V^{(\ell)}\) are message functions, and \(\phi_C^{(\ell)}\) and \(\phi_V^{(\ell)}\) are node-update functions. These functions are usually implemented as multilayer perceptrons with nonlinear activations. Each update depends on both neighboring node embeddings and the corresponding edge weight \(A_{ji}\).

After \(L\) message-passing layers, the predictor applies a variable-level readout to each final variable embedding. For binary integer variables, a common formulation trains a predictor \(f_\theta\) on the graph \(G\) to estimate a target solution \(X^*\):
\begin{equation*}
    \hat{p}_i
    =
    \sigma\!\left(r_\theta(h_i^{(L)})\right)
    =
    f_\theta(G)_i
    \approx P(x_i^*=1 \mid G),
\end{equation*}
where \(r_\theta\) is a learnable readout function and \(\sigma\) is the sigmoid function. The output \(\hat{p}_i\) is the predicted probability that variable \(i\) takes value one in the reference solution. 
The predictions can then be used to construct a partial assignment, define a neighborhood around a candidate solution, or fix variables before invoking the solver on a reduced MILP.

\section{Method}
\label{sec:method}
\subsection{Motivation}

We begin with an observation about the dynamics of modern MILP solvers: the primal and dual gaps typically decrease sharply at the beginning of the solve, after which progress becomes much slower and is spread over a long period of search. Figure~\ref{fig:solver_primal_gap} illustrates this behavior with the evolution of the primal gap on example instances; the same qualitative pattern holds across solvers and problem classes (App.~\ref{app:solver_behavior}). This behavior suggests a two-stage view of MILP solving. In the \emph{fast descent stage}, the solver rapidly discovers useful early solutions and tightens bounds through presolve, heuristics, cutting planes, and early branch-and-bound decisions. In the \emph{long exploration stage}, the solver spends most of its time searching for further improvements, proving bounds, and resolving difficult variables.

This two-stage view motivates a new prediction target for learning-based MILP solving. In contrast to the prevailing paradigm, which constructs a full solution directly from the static MILP graph, we propose an \textbf{E}arly-to-final solutio\textbf{n} \textbf{Co}nsistency p\textbf{re}diction paradigm, termed \textbf{EnCore}. Specifically, given an early solution \(X^{ES}\) collected after the fast-descent stage, we estimate the probability that the early assignment of each variable persists in the full-budget solution \(X^*\). This target distinguishes assignments that have likely stabilized from those that merit further exploration, thereby directing both learning and search effort toward the decisions that remain open after the fast descent.

The proposed paradigm rests on the premise that the early solution is informative of the final one. We now examine this premise empirically. Figure~\ref{fig:early_solution_quality} compares each early solution \(X^{ES}\) with the full-budget solution \(X^*\) by counting the flipped integer variables. Across benchmark classes, the early solutions exhibit agreement with the final ones, and the discrepancies tend to concentrate on a limited subset of variables rather than spreading arbitrarily across the instance. These observations indicate that the early solution carries rich information about the final solution, making it a valuable reference worth exploiting.

\begin{figure}[t]
    \centering
    \begin{subfigure}[t]{0.49\columnwidth}
        \centering
        \includegraphics[width=\linewidth]{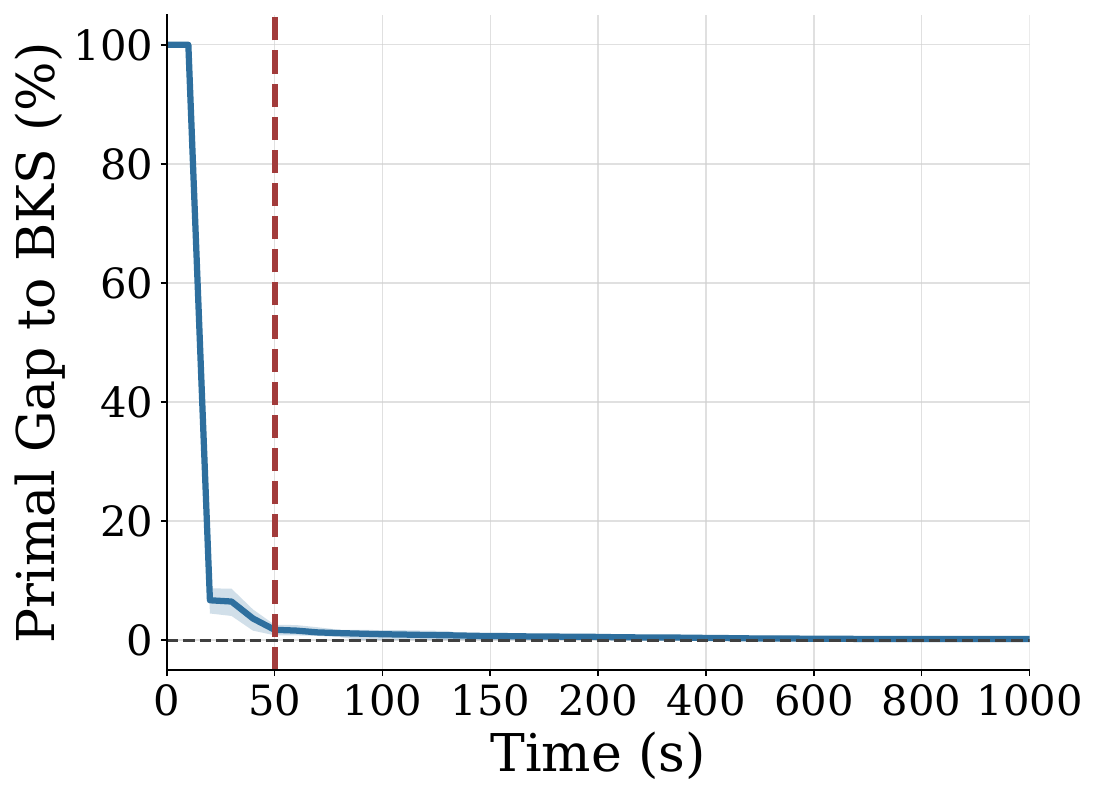}
        \caption{Gurobi Results}
    \end{subfigure}\hfill
    \begin{subfigure}[t]{0.49\columnwidth}
        \centering
        \includegraphics[width=\linewidth]{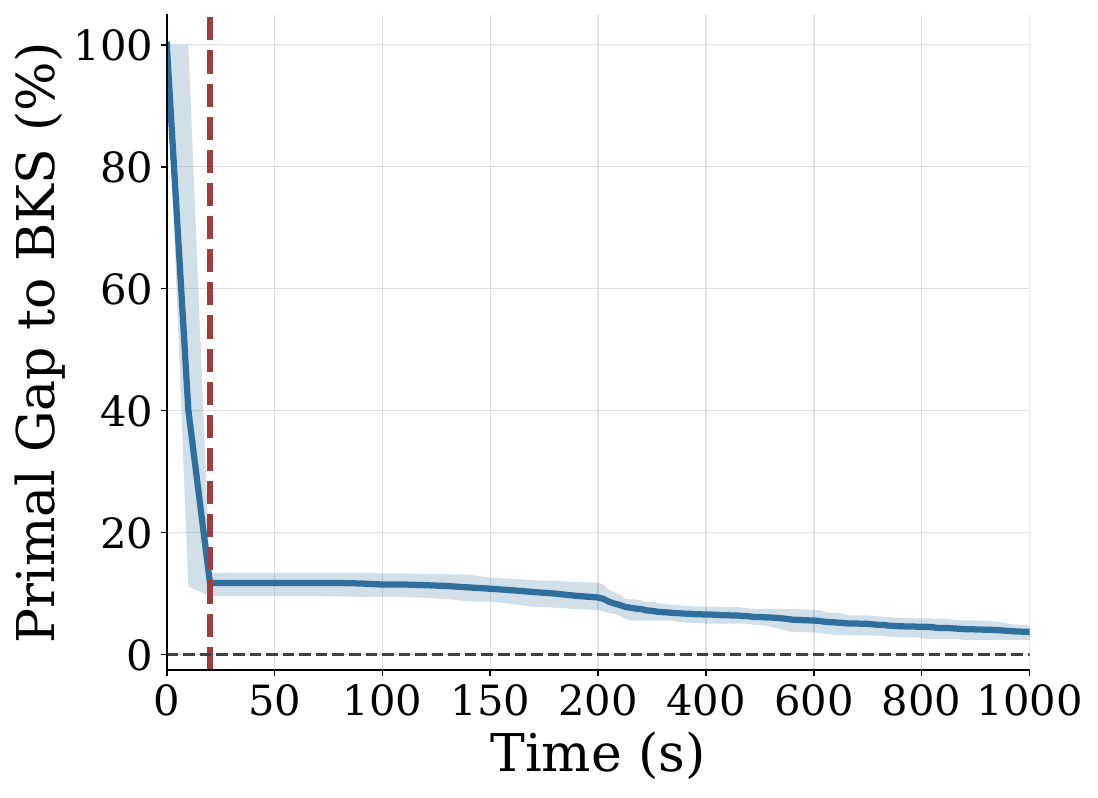}
        \caption{SCIP Results}
    \end{subfigure}\hfill
    \caption{Evolution of the primal gap over the solving process with Gurobi (left) and SCIP (right) on set covering instances. The time axis is scaled to highlight the early phase \(0\!-\!200\,\mathrm{s}\). Results on other problem classes are provided in App.~\ref{app:solver_behavior}.}
    \label{fig:solver_primal_gap}
\end{figure}

\begin{figure}[t]
    \centering
    \begin{subfigure}[t]{0.49\columnwidth}
        \centering
        \includegraphics[width=\linewidth]{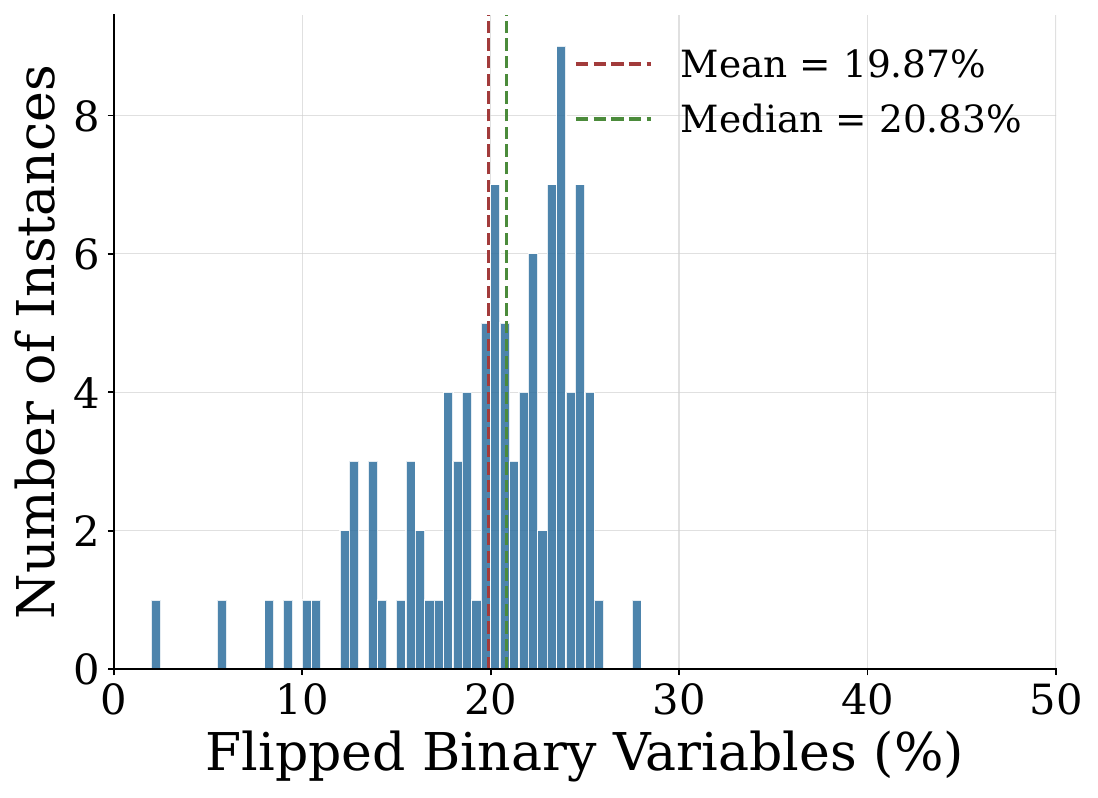}
        \caption{Gurobi on CA}
    \end{subfigure}\hfill
    \begin{subfigure}[t]{0.49\columnwidth}
        \centering
        \includegraphics[width=\linewidth]{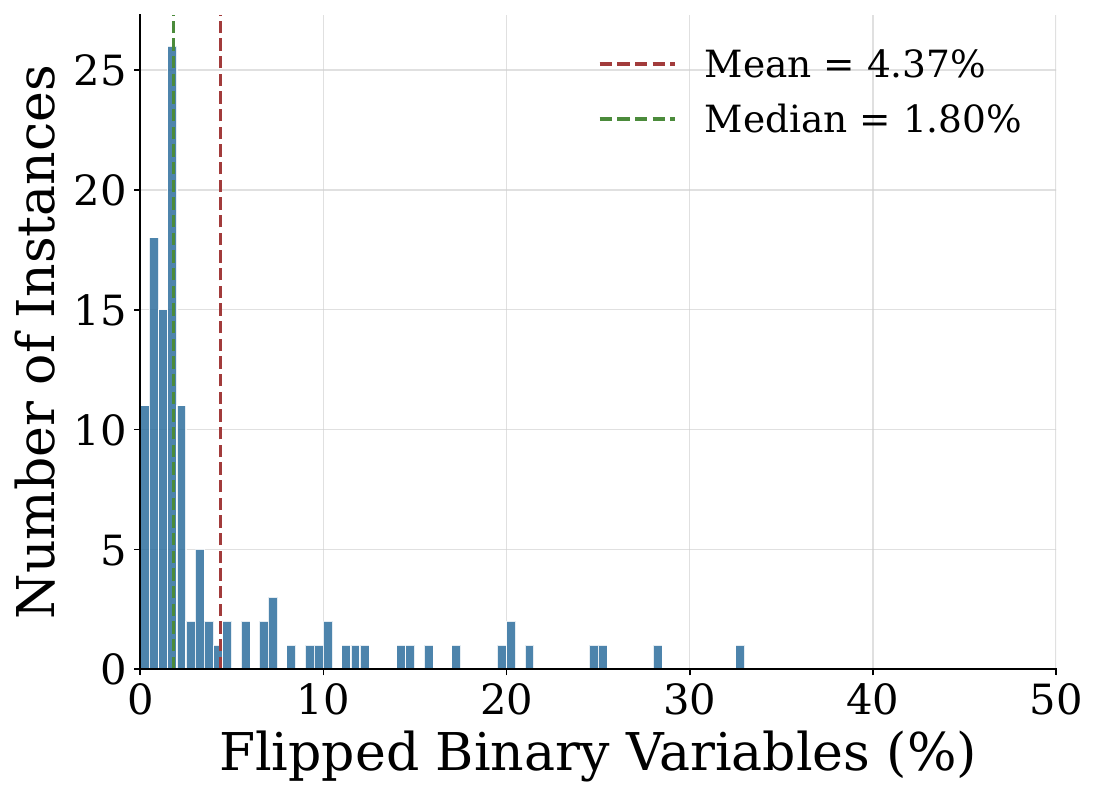}
        \caption{Gurobi on WA}
    \end{subfigure}
    \par\medskip
    \begin{subfigure}[t]{0.49\columnwidth}
        \centering
        \includegraphics[width=\linewidth]{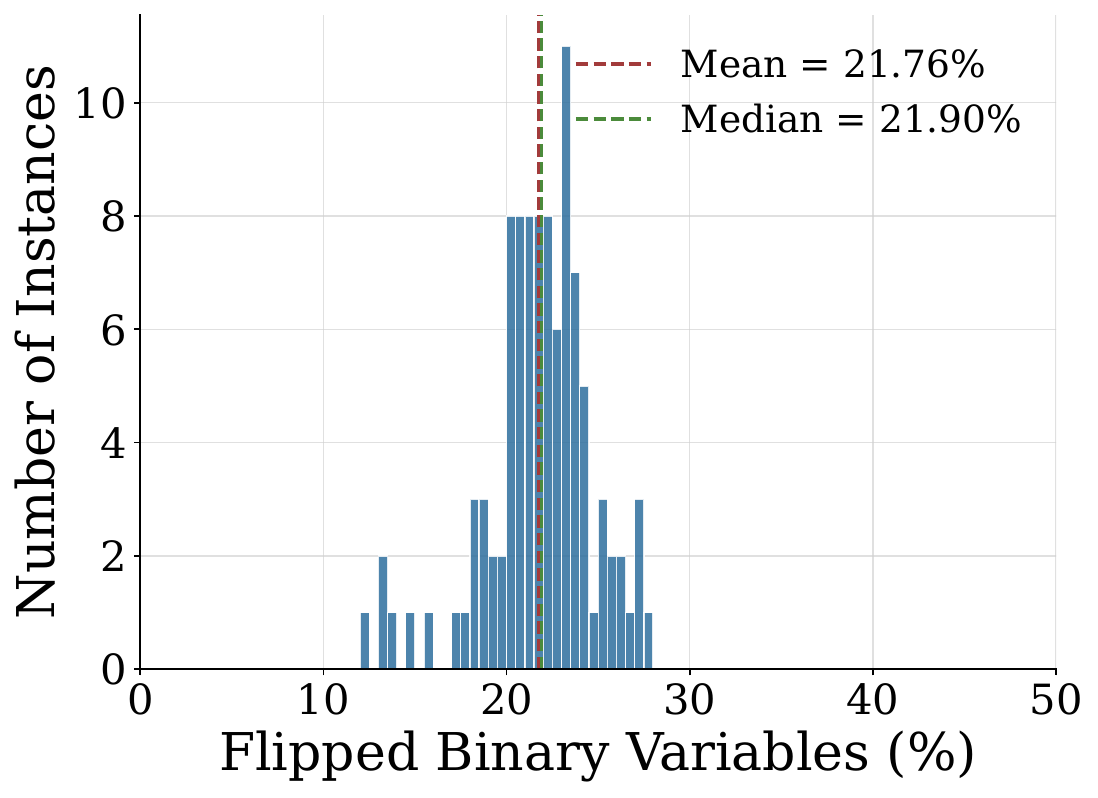}
        \caption{SCIP on CA}
    \end{subfigure}\hfill
    \begin{subfigure}[t]{0.49\columnwidth}
        \centering
        \includegraphics[width=\linewidth]{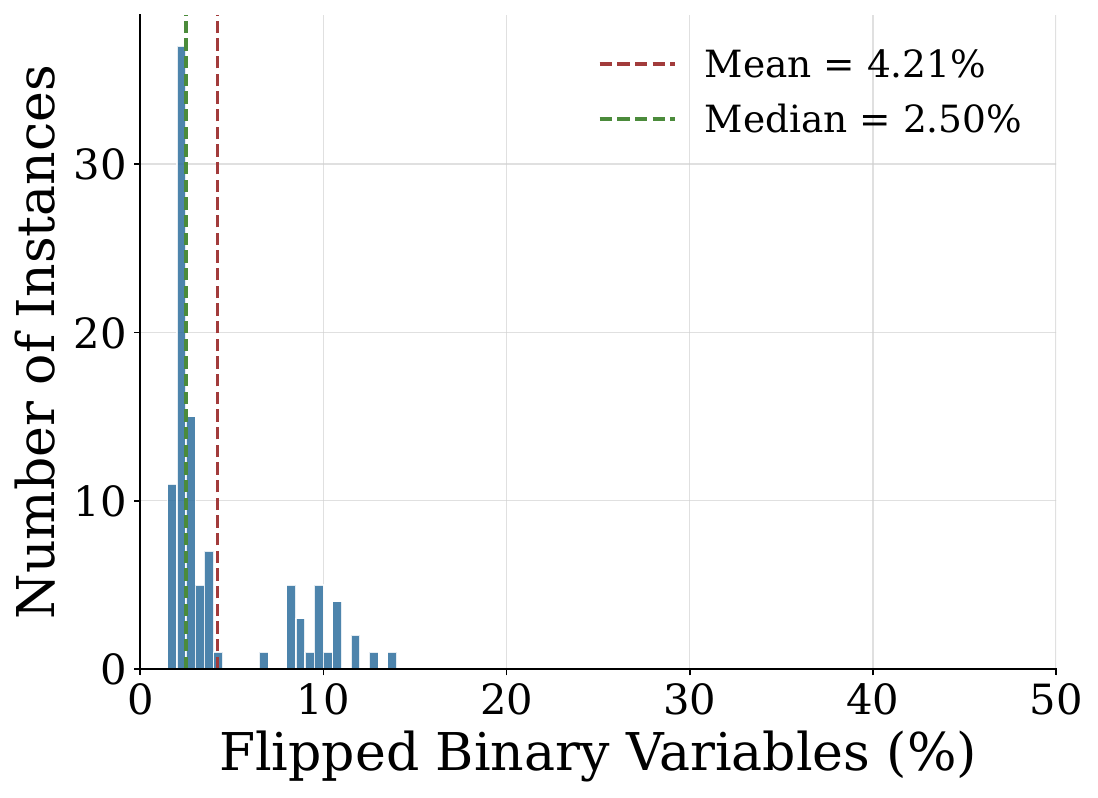}
        \caption{SCIP on WA}
    \end{subfigure}
    \caption{Distributions of flipped integer variables between early and full-budget solutions. `CA' and `WA' stand for combinatorial auction and workload apportionment, respectively.}
    \label{fig:early_solution_quality}
\end{figure}

\subsection{Consistency Prediction}

Given the MILP graph \(G\) and the collected early solution \(X^{ES}\), we design a consistency predictor that estimates, for each integer variable, the probability that its early assignment agrees with the full-budget solution \(X^*\). The predictor must fuse two heterogeneous sources of information: the structural description of the instance, encoded in the bipartite variable–constraint graph, and the variable assignments in the early solution. We achieve this fusion at the input level, where each integer-variable node is augmented with its early assignment $x_i^{ES}$ as an additional feature. Through rounds of variable-to-constraint and constraint-to-variable message passing, the network can thus assess each assignment jointly with the constraints it appears in and the assignments of neighboring variables, which provides the evidence needed to judge whether the assignment will persist. The backbone is a bipartite graph neural network with half-convolutions, a standard architecture for MILP representation learning \cite{han2023gnn,liu2025apollomilp}. A variable-level readout then produces the consistency probability $\hat{p}_i$ for each integer variable. Architectural details are provided in Appendix~\ref{app:model_structure}.

To train the predictor, we fundamentally reshape the learning target, rather than imitating the complete solutions \(X^*\). Since the early solution already encodes substantial information about the final one, the remaining task is not to predict every variable value from scratch, but to assess which early assignments are likely to persist in the final solution. We therefore define the binary consistency label \(y_i = \mathbf{1}[x_i^* = x_i^{ES}]\) for each integer variable \(i\), and train the predictor \(f_\theta\) so that
\begin{equation*}
    \hat{p}_i
    =
    f_\theta(G, X^{ES})_i
    \approx
    P(x_i^*=x_i^{ES}\mid G, X^{ES}).
\end{equation*}

For each training instance \(G\), we collect an early solution \(X^{ES}\) and pair it with the full-budget solution \(X^*\). The collection procedure is described later in this section. The predictor is then trained on the early-to-final consistency labels with the binary cross-entropy loss,
\begin{equation*}
    \mathcal{L}(\theta)
    =
    -\sum_{i\in\mathcal{I}}
    \left[
        y_i\log \hat{p}_i
        +(1-y_i)\log(1-\hat{p}_i)
    \right].
\end{equation*}
This consistency target reshapes the learning task: rather than constructing the full solution from the static graph alone, the model refines an informative reference solution by assessing the stability of each assignment, which is a more focused and better-conditioned problem. Moreover, the conditioning solution \(X^{ES}\) provides a rich source of instance-specific information that purely graph-based predictors cannot access.

\paragraph{Early Solution Collection. }
We collect early solutions at the transition from \textit{fast descent} to \textit{long exploration}. During the solving process, we monitor feasible-solution improvement events and estimate the local dual-gap decay rate over a sliding window, i.e., the gap reduction within the window divided by the elapsed time. The run is terminated once this rate first falls below a threshold, indicating that rapid bound improvement has mostly ended, and the last explored solution is taken as the early solution $X^{ES}$. In practice, we additionally impose a minimum probing time to avoid stopping before a meaningful dual bound, and a maximum time to restrict the runtime cost. Details are included in App.~\ref{app:es_collect}.

\paragraph{Inference-Time Augmentation.}
The early solution available for conditioning is not unique.  Depending on when the collection run terminates, the solver may return slightly different solutions, and a variable predicted stable under one of them may still be predicted fluctuating under another. To mitigate this sensitivity, we employ the last few improving solutions for inference-time augmentation. Let \(\{X_1, \ldots, X_K\}\) denote the last \(K\) improving solutions found during the collection run, with \(X_K = X^{ES}\). The trained predictor processes each \(X_k\) as the conditioning solution. Since a consistency score is always expressed relative to the value taken in its own conditioning solution, the resulting predictions must be aligned to the common reference \(X^{ES}\). To this end, we flip the sign of each logit whose corresponding variable takes a different value in \(X_k\) than in \(X^{ES}\), as a high consistency score for such a variable argues against the reference value.
The aligned logits are then averaged,
\begin{equation*}
    \bar{p}_i
    =
    \sigma\!\left(
        \frac{1}{K}\sum_{k=1}^{K}
        \tau_{k,i}
        \hat{p}_{k,i}
    \right),
\end{equation*}
where \(\tau_{k,i}=+1\) if \(x_{k,i}=x^{ES}_i\) and \(-1\) otherwise. The ensembled probability \(\bar{p}_i\) replaces \(\hat{p}_i\) as the
consistency score of \(x^{ES}_i\) used by the downstream search methods.

\subsection{Combine EnCore with Search Methods}

Our consistency predictor is independent of the downstream search strategy. In most learning-based MILP accelerators, the method must select a subset of variables whose predicted values are trusted, and the search space is then reduced accordingly.
The central question is therefore which assignments can be trusted. Existing methods trust an assignment when the model is confident in its own prediction, so both the value and its evidence come from the model alone. In contrast, our \textit{early-to-final consistency} design trusts values that the solver has already realized in a feasible early solution, and scores them by predicting whether they persist after the full solving process. Specifically, we rank variables by their consistency score \(\bar{p}_i\) and select the top-ranked ones. Each selected variable \(i\) is constrained to its early value \(x_i^{ES}\), while the remaining variables are left to the downstream search mechanism. This consistency-based design offers two main advantages: (1) \textit{Feasibility}. The fixed values are drawn from a feasible solution, so fixing them never destroys feasibility, which model-generated assignments cannot guarantee. (2) \textit{Alignment with the search process}. The consistency signals originate from the solver's own search trajectory, and thus naturally guide the model toward the variables that the solver has yet to resolve.

\section{Theoretical Analysis}

Our method depends on the premise that conditioning on the early solution
$X^{ES}$ makes predicting the full-budget solution $X^*$ fundamentally easier
than predicting from the instance graph alone. In this section, we formalize
it by answering two questions: (1) Can the early solution improve the best achievable prediction accuracy, and under what condition is the gain strict
(\textbf{Theorem~\ref{thm:info-gain}})? (2) Does this gain survive when the
predictor must be selected from finite training instances
(\textbf{Theorem~\ref{thm:finite-sample}})? Complete proofs are in Appendix~\ref{app:theoretical_proofs}.

Let $\mathcal{D}$ be the distribution over triples $(G, X^{ES}, X^*)$, and let
$J$ be a uniformly weighted variable index within an instance, so that
$\Pr(\widehat{x}_J = x^*_J)$ is the expected per-instance variable
accuracy, i.e., the population variable-level evaluation criterion used in our analysis. A message-passing
predictor at variable $J$ considers only the instance features in its receptive MILP graph
$\mathcal{R}_J$, while EnCore additionally attaches early assignments $X^{ES}_{\mathcal{R}_J}$ to construct the augmented input $\mathcal{U}_J = (\mathcal{R}_J, X^{ES}_{\mathcal{R}_J})$.

A solution predictor $h$ maps $\mathcal{R}_J$ to a prediction of $x^*_J$,
matching the conventional paradigm. A consistency predictor $a$
observes $\mathcal{U}_J$ and decides whether to retain the early value, inducing
\begin{equation*}
  h_a(\mathcal{U}_J) =
  \begin{cases}
    x^{ES}_J,     & a(\mathcal{U}_J) = 1,\\
    1 - x^{ES}_J, & a(\mathcal{U}_J) = 0.
  \end{cases}
\end{equation*}
Given $x^{ES}_J$, consistency decisions and final assignments determine each
other~(see Appendix~\ref{app:theoretical_proofs}), so the two paradigms can be compared as
predictors of $x^*_J$.

\paragraph{Information gain of the early solution.}
Define the best population accuracies attainable from each input,
\begin{align*}
  A^*_{\mathrm{sol}} &:= \sup_h \Pr\bigl(h(\mathcal{R}_J) = x^*_J\bigr),
  \\
  A^*_{\mathrm{con}} &:= \sup_a \Pr\bigl(h_a(\mathcal{U}_J) = x^*_J\bigr).
\end{align*}
If $\mathcal{D}$ were known, these would be achieved by the Bayes-optimal
rules on each input \citep{devroye1996probabilistic}. Their comparison isolates the informational
value of the early solution from any modeling concern.

\begin{theorem}
\label{thm:info-gain}
For any fixed early-solution collection procedure,
$A^*_{\mathrm{con}} \ge A^*_{\mathrm{sol}}$. Let
$\eta_J := \Pr(x^*_J = 1 \mid \mathcal{U}_J)$. The inequality is strict if
\begin{equation*}
    \Pr\!\left(
        \begin{aligned}
            &\Pr(\eta_J<\tfrac12\mid\mathcal R_J)>0\\
            &\text{and }\Pr(\eta_J>\tfrac12\mid\mathcal R_J)>0
        \end{aligned}
    \right)>0.
\end{equation*}
\end{theorem}

\begin{proof}[Proof sketch]
Every solution predictor $h$ defines
$a_h(\mathcal{U}_J) = \mathbf{1}\{h(\mathcal{R}_J) = x^{ES}_J\}$, for which
$h_{a_h}(\mathcal{U}_J) = h(\mathcal{R}_J)$. Hence the augmented input can
reproduce every solution predictor. For strictness, conditional Jensen's
inequality \citep{kallenberg2021foundations} shows that refining the input
strictly increases Bayes accuracy under the stated condition.
\end{proof}

The inequality itself is unsurprising since additional input can never hurt a
Bayes-optimal predictor. The substantive content is about the strictness condition, which we call
\emph{posterior crossing}: for a non-negligible set of instance inputs,
different early assignments compatible with the same instance structure favor
different final values. In this regime, no predictor on instance features
alone, however expressive, can resolve the ambiguity, while the early
solution can. Figure~\ref{fig:early_solution_quality} provides
complementary empirical motivation by showing that early-to-final
discrepancies are often concentrated on a small fraction of variables.

\paragraph{A sparse-correction reading of the gain.}
Let $a^*$ be the Bayes rule on the augmented input,
$p = \Pr(x^{ES}_J \ne x^*_J)$ represents the early-solution error rate,
\(\rho=\Pr(a^*(\mathcal U_J)=0)\) denotes its flipping rate, and $q$ is the precision on flipped variables, where $q>1/2$ due to the Bayes optimality. Then we have
\begin{align*}
  1 - A^*_{\mathrm{con}} &= p - \rho(2q - 1),\\
  A^*_{\mathrm{con}} - A^*_{\mathrm{sol}}
    &= 1 - A^*_{\mathrm{sol}} - p + \rho(2q - 1).
\end{align*}
The gain is positive exactly when
$p < 1 - A^*_{\mathrm{sol}} + \rho(2q-1)$. This identity captures the residual
nature of EnCore: corrections with \(q>1/2\) remove more errors than they
introduce, and when early-to-final discrepancies are sparse, learning can
focus on a targeted subset of unstable assignments. The condition can also
hold when \(1-p<A^*_{\mathrm{sol}}\) (see
Appendix~\ref{app:theoretical_proofs}), showing that the early solution is
valuable not only as a feasible candidate but also as an informative
conditioning signal.

\paragraph{Finite-sample guarantee.}
Theorem~\ref{thm:info-gain} optimizes over all rules and thus characterizes
only the information available in principle. However, the predictor is
selected from finite data in practice, and one may worry whether the population
gain can survive estimation errors. Following the standard finite-class
empirical-risk-minimization analysis based on Bernstein's inequality and a
union bound \citep{boucheron2013concentration}, we fix a finite class
$\mathcal{A} = \{a_1, \dots, a_N\}$ of consistency rules before observing $m$
independent training instances, and let $\widehat{a}$ minimize the average
proportion of variable errors; only instances need be independent, matching
how MILP training data are collected. The class advantage decomposes as
$\Delta_{\mathcal{A}} = \Delta_{\mathrm{B}} - \alpha_{\mathcal{A}}$, where
$\Delta_{\mathrm{B}} := A^*_{\mathrm{con}} - A^*_{\mathrm{sol}}$ is the gain
from Theorem~\ref{thm:info-gain} and
$\alpha_{\mathcal{A}} := A^*_{\mathrm{con}} - \max_{a \in \mathcal{A}} \Acc(h_a)$
is the approximation gap, with $\Acc(h) = \Pr(h_J = x^*_J)$. Motivated by the
sparse discrepancies in Figure~\ref{fig:early_solution_quality}, we assume the
expected per-instance flip rate satisfies
\(\max_{a\in\mathcal{A}}\mathbb{E}[s_a]\le\bar{s}\).

\begin{theorem}
\label{thm:finite-sample}
Let $\widehat{h}_{ES} = h_{\widehat{a}}$, and let $\widehat{h}_{\mathrm{sol}}$
be any possibly data-dependent solution predictor. For every
$\delta \in (0,1)$, with probability at least $1 - \delta$,
\begin{equation*}
\begin{aligned}
  \Acc(\widehat{h}_{ES}) - \Acc(\widehat{h}_{\mathrm{sol}})
  \ge \Delta_{\mathrm{B}}
  - \alpha_{\mathcal{A}}
  - \varepsilon_m(\delta)&, \\
  \varepsilon_m(\delta)
  = 2\sqrt{\frac{2\bar{s}}{m}
      \ln \frac{2N}{\delta}}
    + \frac{8}{3m}
      \ln \frac{2N}{\delta}&.
\end{aligned}
\end{equation*}
\end{theorem}

\begin{proof}[Proof sketch]
Use the rule that always keeps the early solution as a common reference. A
candidate differs from it only on variables it flips, so its instance-level
loss difference has variance at most $\bar{s}$. Bernstein's inequality with a
union bound over the $N$ rules \citep{boucheron2013concentration} then
controls all empirical loss differences simultaneously. Empirical risk minimization converts this into the penalty $\varepsilon_m(\delta)$, while Bayes optimality lower-bounds the risk of every instance-input
solution predictor.
\end{proof}

\vspace{-0.5em}
Thus the learned consistency predictor provably outperforms \emph{any} solution
predictor, even one trained with unlimited data, whenever
$\Delta_{\mathrm{B}} > \alpha_{\mathcal{A}} + \varepsilon_m(\delta)$. The
dominant term of $\varepsilon_m(\delta)$ scales as $\sqrt{\bar{s}/m}$ rather
than $\sqrt{1/m}$: sparse correction rules are cheaper to select reliably. A
numerical illustration in Appendix~\ref{app:theory_example} shows that under representative
parameters, fewer than $200$ training instances suffice, which is similar to the typical sizes of training datasets.

In summary, Theorem~\ref{thm:info-gain} establishes that the early solution
provides a strict population-level information gain under a stated
posterior-crossing condition, and the sparse-correction identity explains
where the gain comes from: repairing a small, better-than-chance set of
unstable assignments. Theorem~\ref{thm:finite-sample} shows that this gain
persists under finite-sample model selection, with a sample cost discounted by
the very sparsity that motivates our design. Together, they justify shifting
the learning target from full solution prediction to early-to-final
consistency estimation.
\vspace{-5em}
\FloatBarrier
\section{Experiments}
\label{sec:experiments}

\subsection{Experimental Setup}
\begin{table*}[htbp]
    \centering
    \small
    \setlength{\tabcolsep}{1mm}
    \begin{tabular}{@{}lrrrrrrrr@{}}
        \toprule
        & \multicolumn{2}{c}{CA $\uparrow$ (BKS 98627.99)}
        & \multicolumn{2}{c}{SC $\downarrow$ (BKS 123.37)}
        & \multicolumn{2}{c}{WA $\downarrow$ (BKS 706.86)}
        & \multicolumn{2}{c}{IP $\downarrow$ (BKS 11.72)} \\
        \cmidrule(lr){2-3}\cmidrule(lr){4-5}
        \cmidrule(lr){6-7}\cmidrule(lr){8-9}
        Method & Obj & Gap & Obj & Gap & Obj & Gap & Obj & Gap \\
        \midrule
        Gurobi (3600s) & 98448.84 & 179.15 & 123.37 & 0.00 & 706.86 & 0.00 & 11.72 & 0.00 \\
        Gurobi (1000s) & 97311.69 & 1316.30 & 123.64 & 0.27 & 707.36 & 0.50 & 13.77 & 2.05 \\
        \midrule
        ND & 94340.63 & 4287.36 & 123.62 & 0.25 & 707.10 & 0.24 & 14.15 & 2.43 \\
        PS & 97906.20 & 721.79 & 123.60 & 0.23 & 707.09 & 0.23 & 12.08 & 0.36 \\
        Apollo & 98083.79 & 544.20 & 123.56 & 0.19 & 707.06 & 0.20 & 11.97 & 0.25 \\
        \midrule
        \ours{}-ND & 97847.92 & 780.70 & 123.58 & 0.21 & 707.03 & 0.17 & 13.47 & 1.75 \\
        \ours{}-PS & \textbf{98627.99} & \textbf{0.00} & \textbf{123.50} & \textbf{0.13}
        & \textbf{706.98} & \textbf{0.12} & 11.95 & 0.23 \\
        \ours{}-Apollo & 98491.40 & 136.59 & 123.57 & 0.20 & 707.03 & 0.17
        & \textbf{11.82} & \textbf{0.10} \\
        \midrule
        Best Gap Reduction & 721.79 & \textbf{100.0\%} & 0.10 & \textbf{43.5\%}
        & 0.11 & \textbf{47.8\%} & 0.15 & \textbf{60.0\%} \\
        \bottomrule
       
    \end{tabular}
     \caption{Main results with Gurobi. Obj is the final feasible objective and
    Gap is the absolute gap to the in-study BKS. Arrows indicate the preferred
    objective direction; bold marks the best learning-based result or a
    positive gap reduction.}
    \label{tab:gurobi_main_results}
\end{table*}

\begin{figure*}[ht]
    \centering
    \includegraphics[width=\textwidth]{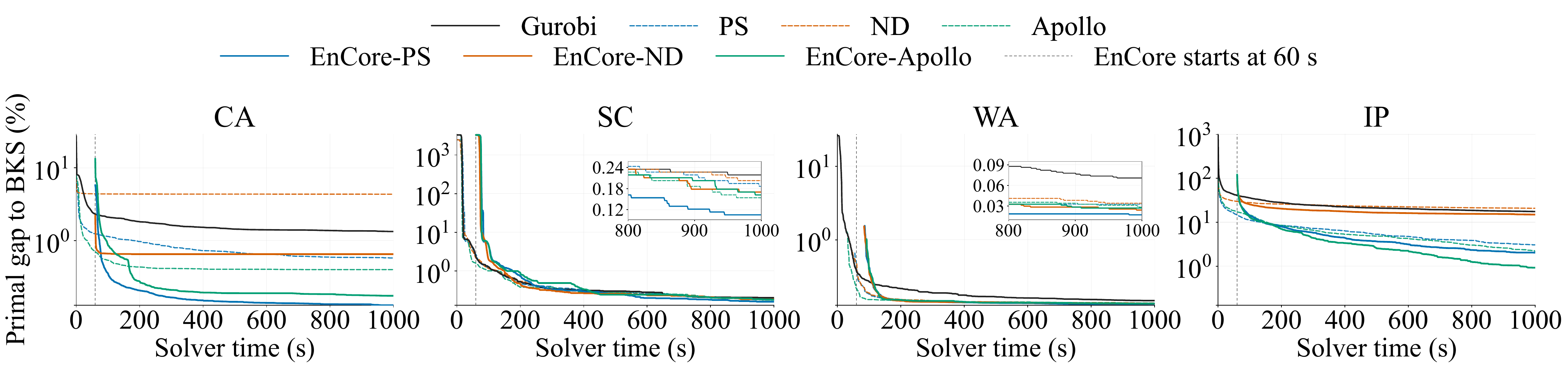}
    \caption{Average primal gap to the BKS versus time under a 1,000-second time limit. \ours{} starts after the early solution collection. Each curve is shown only after all test instances have obtained a feasible solution.}
    \label{fig:main_primal_gap}
\end{figure*}

\paragraph{Benchmarks.}
Following \citet{liu2025apollomilp}, we evaluate \ours{} on Combinatorial Auctions (CA)~\cite{leytonbrown2000auction}, Set Covering (SC)~\cite{balas1980setcovering}, Workload Apportionment (WA), and Item Placement (IP)~\cite{gasse2021ml4co}. These benchmarks cover MILPs with diverse scales and structures. For zero-shot transfer, we additionally evaluate on the eleven-instance MIPLIB IIS benchmark considered by \citet{liu2025apollomilp}, following the IIS setting introduced by \citet{wang2024digmilp}. The benchmark is a structurally related subset of MIPLIB centered on binary set-covering formulations, while exhibiting substantial variation in problem size, constraint structure, and sparsity. Dataset sources and instance statistics are reported in App.~\ref{app:experimental_details}.

\paragraph{Baselines and variants.}
We compare against the native MILP solver and three prediction-guided search pipelines: \nd{} (ND)~\cite{nair2020solving}, \ps{} (PS)~\cite{han2023gnn}, and Apollo-MILP (Apollo)~\cite{liu2025apollomilp}. 
For each learning-based pipeline, we replace its original solution predictor with our consistency predictor while keeping other configurations; the resulting methods are denoted as \ours{}-ND, \ours{}-PS, and \ours{}-Apollo. 
This design isolates the contribution of the proposed consistency prediction mechanism from downstream solver policies.

\paragraph{Evaluation protocol.}
All methods use a 1,000-second end-to-end budget that includes early-solution collection, inference, and downstream solving. For a fair comparison, \ours{} does not use the collected early solutions to warm-start the final search. For \ours{}-Apollo, early solutions are collected once and reused across correction rounds. We report the final feasible objective, its absolute gap to the best-known solution (BKS), and the gap reduction relative to the corresponding solution predictor baseline. The BKS is the best mean final objective among all evaluated methods, including the 3,600-second Gurobi reference.

\subsection{Main Results}

Table~\ref{tab:gurobi_main_results} reports the final objectives under Gurobi. Substituting EnCore for the original solution predictor reduces the final gap in 11 of the 12 settings (4 benchmarks 
$\times$ 3 search mechanisms), by 38.7\%, 56.9\%, and 36.2\% on average for ND, PS, and Apollo, respectively. The sole exception is Apollo on SC, where the gap grows marginally from 0.19 to 0.20. On CA, \ours{}-PS finds the best solution among all compared methods, improving on the 3,600-second Gurobi incumbent by 179.15 within a 1,000-second budget, and \ours{}-Apollo also exceeds this reference (+42.56). On SC, WA, and IP, where the baselines have already been close to the BKS, \ours{} still removes 43.5\%, 47.8\%, and 60.0\% of the best baseline gap. Remarkably, \ours{} remains effective even under weak easy-to-final consistency. On CA, where 19\% of the variables flip between the early and final solutions, the consistency predictor still captures the learnable patterns underlying these flips, allowing \ours{} to achieve a 100\% average gap reduction.

Figure~\ref{fig:main_primal_gap} traces the primal gap over time.
Although the best \ours{} variant on each benchmark begins prediction-guided
search only after the 60-second collection stage, it catches up with its
corresponding baseline at 71, 241, 164, and 93~seconds of end-to-end runtime
on CA, SC, WA, and IP, respectively, and remains ahead thereafter.
\ours{}-PS overtakes all competitors early on CA and attains the lowest gap
on SC and WA, with the late-stage separation shown in the insets. On IP,
\ours{}-Apollo exhibits the strongest late-stage anytime performance and
widens its lead over time.

\paragraph{Zero-shot cross-solver transfer to SCIP.}
We directly apply the Gurobi-trained checkpoint to SCIP-generated trajectories without retraining or model selection. As shown in Table~\ref{tab:scip_results}, \ours{}-ND consistently improves all four benchmarks, reducing the absolute gap by 53.6\%, 22.0\%, 33.1\%, and 50.7\% on CA, SC, WA, and IP, respectively. Moreover, \ours{}-PS also improves performance on CA, WA, and IP. These results show that the learned consistency representation remains informative under a solver shift.
\begin{table}[ht]
\centering
\small
\setlength{\tabcolsep}{2mm}
\begin{tabular}{@{}lrrrr@{}}
\toprule
Method
& CA $\uparrow$
& SC $\downarrow$
& WA $\downarrow$
& IP $\downarrow$ \\
\midrule
SCIP(1000s)
& 94701.91
& 127.99
& 709.06
& 23.49 \\
\midrule
ND
& 94342.12
& 125.14
& 707.33
& 18.19 \\
PS
& 97097.46
& 125.03
& 708.49
& 17.02 \\
Apollo
& 97335.92
& 124.97
& 708.43
& 16.21 \\
\midrule
\ours{}-ND
& 96641.43
& \textbf{124.75}
& \textbf{707.24}
& \textbf{14.91} \\
\ours{}-PS
& \textbf{97707.64}
& 125.04
& 708.22
& 16.13 \\
\ours{}-Apollo
& 97293.51
& 125.38
& 707.91
& 17.39 \\
\midrule
Best Gap Reduction
& \textbf{53.6\%}
& \textbf{22.0\%}
& \textbf{33.1\%}
& \textbf{50.7\%} \\
\bottomrule
\end{tabular}
\caption{Zero-shot transfer from Gurobi to SCIP under a 1,000-second total
budget. 
Bold marks the best result. 
The BKS is presented in Table~\ref{tab:gurobi_main_results}.}
\label{tab:scip_results}
\end{table}

\paragraph{Zero-shot cross-family transfer to MIPLIB IIS.}
We train our EnCore models on the mixed SC, CA, WA, and IP instances, and evaluate them on the eleven unseen MIPLIB instances used by \citet{liu2025apollomilp} under the same 1,000-second budget. The clearest advantage of EnCore is about feasibility. With ND, the original predictor finds feasible solutions on only 3 of 11 instances, while \ours{}-ND succeeds on all 11. With PS and Apollo, where the baselines already achieve full feasibility, \ours{} still yields consistent improvements in mean objectives. These results demonstrate that our proposed EnCore exhibits greater reliability and robustness than existing methods when faced with substantial distribution shifts. Details are provided in Appendix~\ref{app:experimental_details}.

\subsection{Ablation Study}

Table~\ref{tab:ablation_study} presents a cumulative ablation of the Predict-and-Search pipeline. Adding early solutions as input features alone improves most benchmarks. However, the consistency target contributes substantially larger gains than the extra features themselves, indicating that reformulating the prediction objective is empirically effective. Finally, introducing early-solution ensembling further improves CA, SC, and WA, delivering a smaller but complementary robustness benefit.

\begin{table}[ht]
    \centering
    \small
    \setlength{\tabcolsep}{2.3pt}
    \begin{tabular}{@{}lrrrr@{}}
        \toprule
        Variant & CA $\uparrow$ & SC $\downarrow$ & WA $\downarrow$ & IP $\downarrow$ \\
        \midrule
        \ps & 97906.20 & 123.60 & 707.09 & 12.08 \\
        + early solution as feature & 98318.22 & 123.55 & 707.04 & 12.17 \\
        + consistency prediction target & 98616.65 & 123.53
        & 706.99 & \textbf{11.90} \\
        + early-solution ensemble & \textbf{98627.99} & \textbf{123.50}
        & \textbf{706.98} & 11.95 \\
        \bottomrule
    \end{tabular}
    \caption{Ablation study of key components under the Predict-and-Search pipeline. Average objective values are reported.}
    \label{tab:ablation_study}
\end{table}

\vspace{-0.5em}
\paragraph{Trade-off Analysis of Early-Solution Collection}
\label{sec:trade-off}
Figure~\ref{fig:collection_time_ablation} varies the collection budget
$T_{\max}$ under the fixed 1,000-second total budget. Across all benchmarks,
performance peaks at a small $T_{\max}$: a short collection stage suffices to obtain informative early
solutions, whereas larger $T_{\max}$ leads to diminishing improvements in early-solution quality while consuming more time budget for downstream search.
\begin{figure}[ht]
    \centering
    \includegraphics[width=0.85\columnwidth]{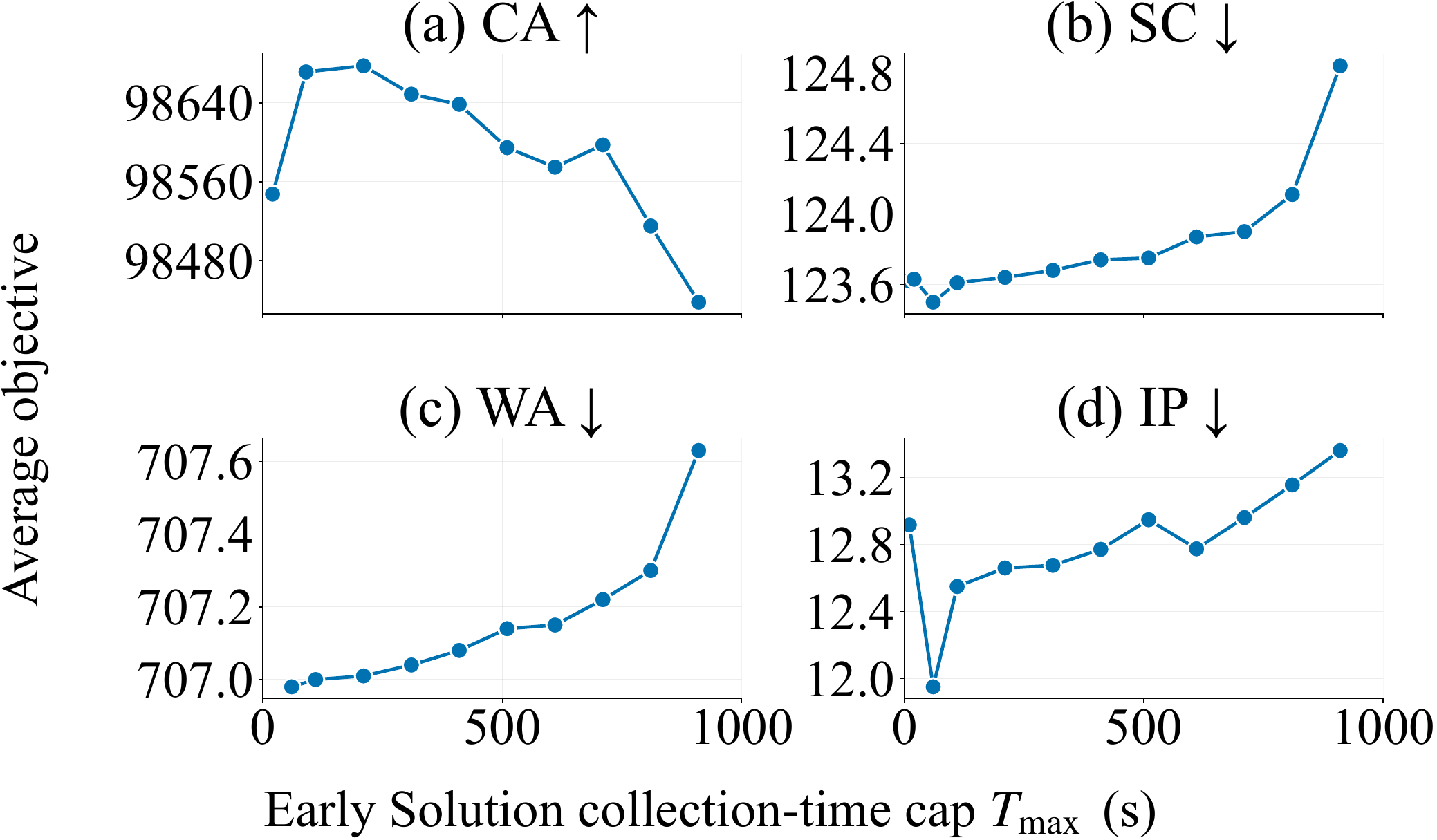}
    \caption{Average final objective of \ours{}-PS under different maximum early-solution collection times.}
    \label{fig:collection_time_ablation}
\end{figure}

\section{Conclusion}

In this paper, we presented EnCore, a solver-informed paradigm that accelerates MILP solving via early-to-final solution consistency prediction. Instead of constructing complete solutions from the static MILP formulation alone, EnCore conditions on feasible solutions collected during the early search stage and identifies the assignments likely to persist under a full solving budget. The resulting consistency scores integrate seamlessly with diverse prediction-guided search frameworks, and ensembling multiple early solutions further enhances robustness. Theoretically, we showed that conditioning on early solutions yields a strict gain in achievable prediction accuracy, and that this gain survives finite-sample model selection at a cost discounted by correction sparsity. Empirically, EnCore consistently improves prediction-guided solving across four benchmarks and generalizes zero-shot to unseen solvers and problem families.
Future work will explore jointly learning consistency prediction and downstream search policies.
\newpage
\bibliography{aaai2027}

% Check whether the conference requires a reproducibility checklist to be included in the paper.
% If so, you can uncomment the following line and ajust the path to include it.
% \input{ReproducibilityChecklist.tex}
% \newpage
\appendix
\setcounter{secnumdepth}{1} %May be changed to 1 or 2 if section numbers are desired.
\section{Solver Behavior}
\label{app:solver_behavior}

Figures~\ref{fig:gurobi_solver_behavior}
and~\ref{fig:scip_solver_behavior} show the primal gap versus solving time
for Gurobi and SCIP on the four benchmarks. Figures~\ref{fig:gurobi_early_solution_all}
and~\ref{fig:scip_early_solution_all} show the distributions of integer
variables flipped between the early and full-budget solutions.
Figure~\ref{fig:miplib_early_solution_all} reports the same analyses on MIPLIB.

\begingroup
\setlength{\floatsep}{5pt plus 1pt minus 1pt}
\setlength{\textfloatsep}{5pt plus 1pt minus 1pt}
\setlength{\intextsep}{5pt plus 1pt minus 1pt}
\captionsetup[figure]{skip=2pt}
\captionsetup[subfigure]{skip=0pt}
\setcounter{topnumber}{5}
\setcounter{bottomnumber}{5}
\setcounter{totalnumber}{10}
\renewcommand{\topfraction}{0.99}
\renewcommand{\bottomfraction}{0.99}
\renewcommand{\textfraction}{0.01}
\newcommand{\solverpanelwidth}{0.42}

\begin{figure}[!ht]
    \centering
    \begin{subfigure}[t]{\solverpanelwidth\columnwidth}
        \centering
        \includegraphics[width=\linewidth]{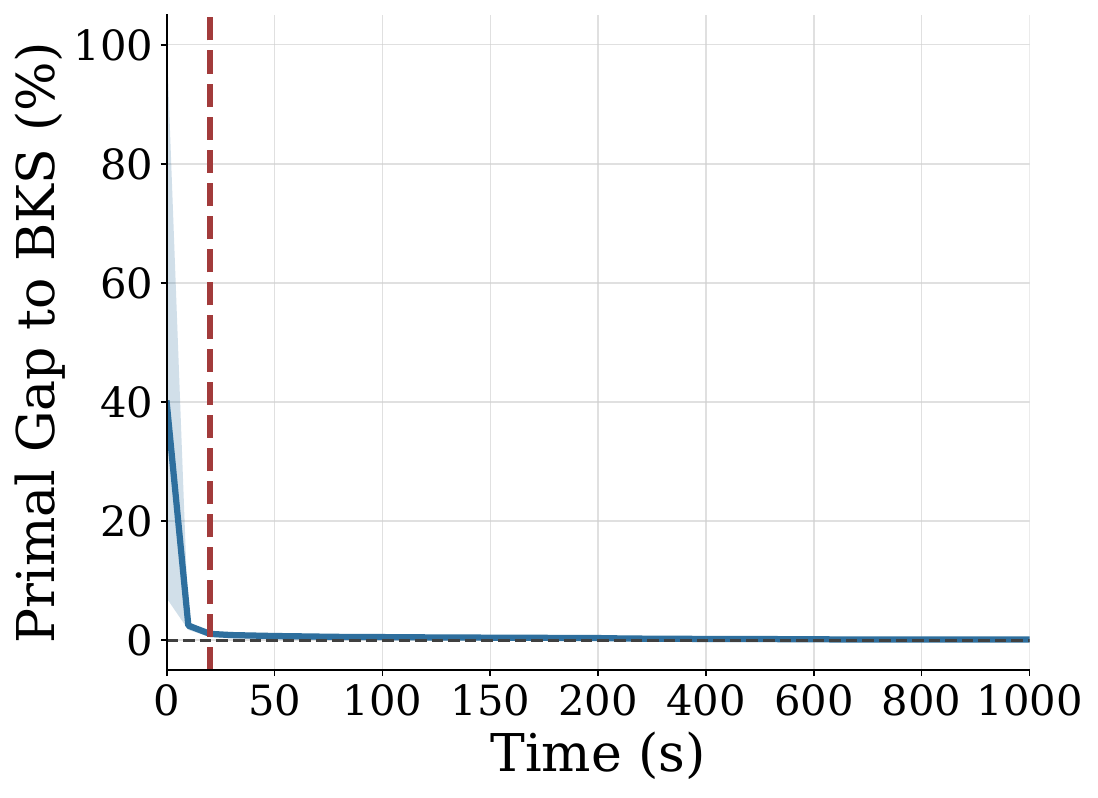}
        \caption{CA}
        \label{fig:gurobi_behavior_ca}
    \end{subfigure}
    \hfill
    \begin{subfigure}[t]{\solverpanelwidth\columnwidth}
        \centering
        \includegraphics[width=\linewidth]{Figures/gap_info/SC_test_gurobi/primal_gap_vs_time.pdf}
        \caption{SC}
        \label{fig:gurobi_behavior_sc}
    \end{subfigure}

    \begin{subfigure}[t]{\solverpanelwidth\columnwidth}
        \centering
        \includegraphics[width=\linewidth]{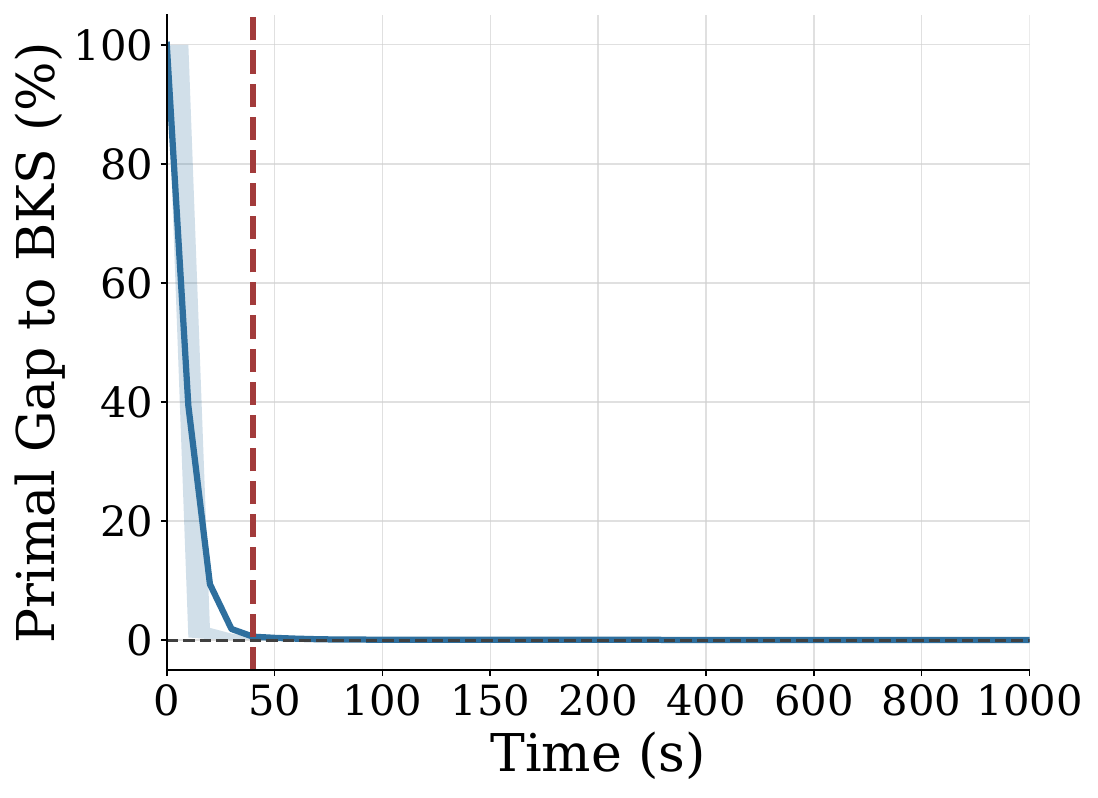}
        \caption{WA}
        \label{fig:gurobi_behavior_wa}
    \end{subfigure}
    \hfill
    \begin{subfigure}[t]{\solverpanelwidth\columnwidth}
        \centering
        \includegraphics[width=\linewidth]{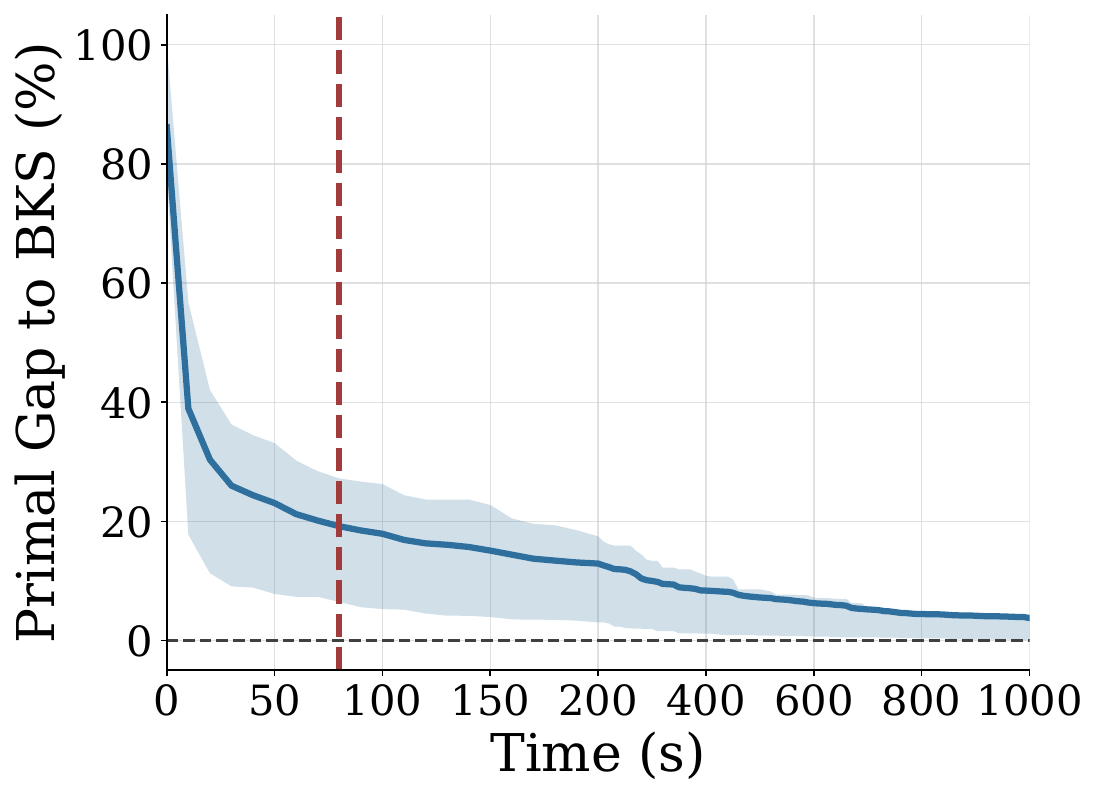}
        \caption{IP}
        \label{fig:gurobi_behavior_ip}
    \end{subfigure}
    \caption{Primal gap versus solving time for Gurobi.}
    \label{fig:gurobi_solver_behavior}
\end{figure}

\begin{figure}[!ht]
    \centering
    \begin{subfigure}[t]{\solverpanelwidth\columnwidth}
        \centering
        \includegraphics[width=\linewidth]{Figures/early_solution_quality/CA_test_flip_count_vs_best.pdf}
        \caption{CA}
    \end{subfigure}
    \hfill
    \begin{subfigure}[t]{\solverpanelwidth\columnwidth}
        \centering
        \includegraphics[width=\linewidth]{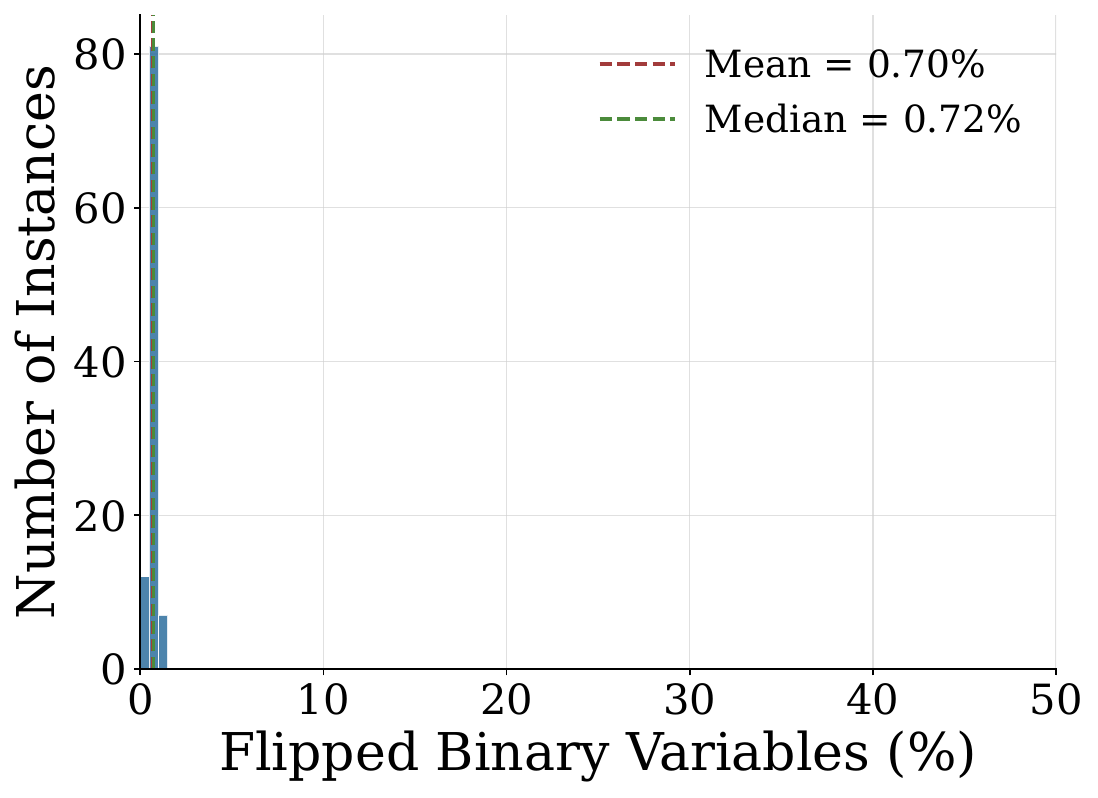}
        \caption{SC}
    \end{subfigure}

    \begin{subfigure}[t]{\solverpanelwidth\columnwidth}
        \centering
        \includegraphics[width=\linewidth]{Figures/early_solution_quality/WA_test_flip_count_vs_best.pdf}
        \caption{WA}
    \end{subfigure}
    \hfill
    \begin{subfigure}[t]{\solverpanelwidth\columnwidth}
        \centering
        \includegraphics[width=\linewidth]{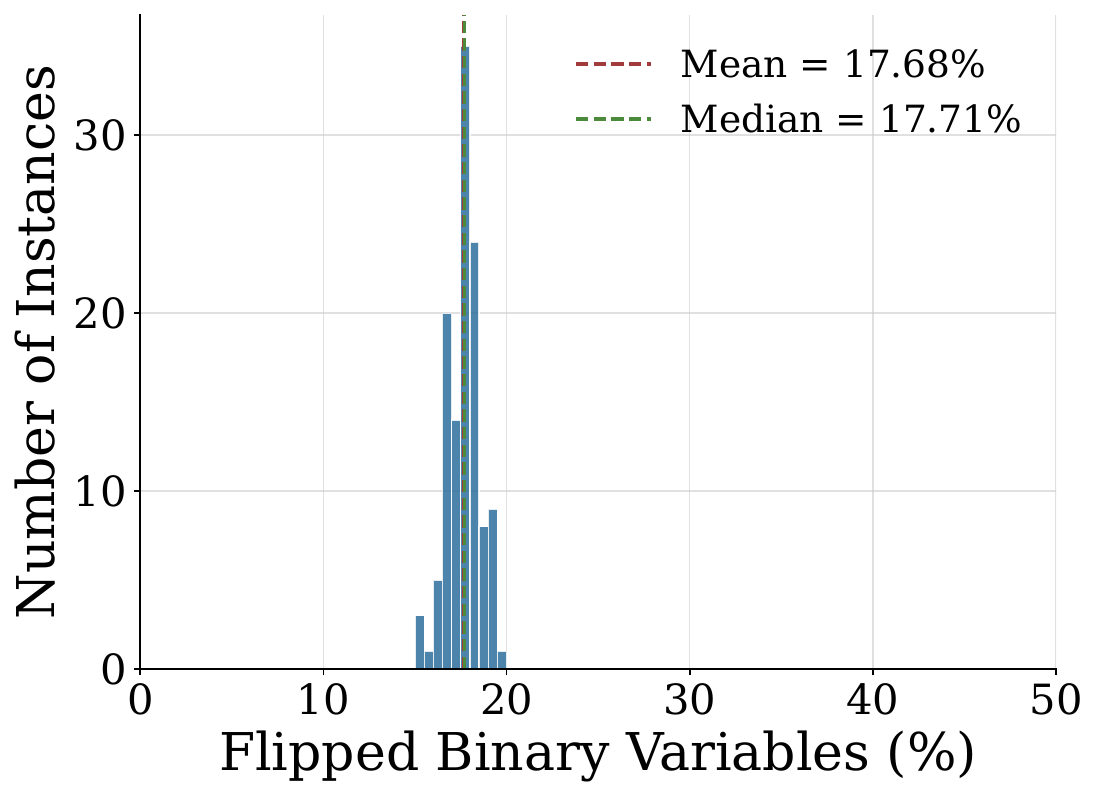}
        \caption{IP}
    \end{subfigure}
    \caption{Early-to-final flip distributions for Gurobi; red and green dashed lines denote the mean and median.}
    \label{fig:gurobi_early_solution_all}
\end{figure}

\begin{figure}[!ht]
    \centering
    \begin{subfigure}[t]{\solverpanelwidth\columnwidth}
        \centering
        \includegraphics[width=\linewidth]{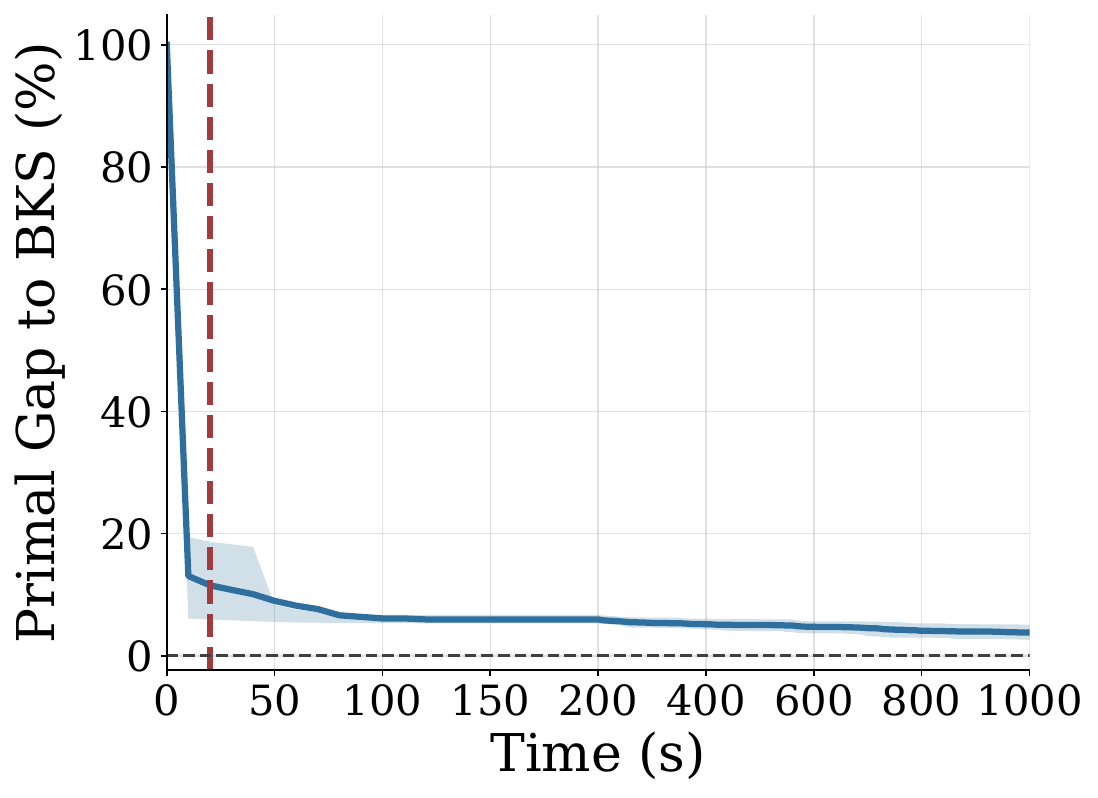}
        \caption{CA}
        \label{fig:scip_behavior_ca}
    \end{subfigure}
    \hfill
    \begin{subfigure}[t]{\solverpanelwidth\columnwidth}
        \centering
        \includegraphics[width=\linewidth]{Figures/gap_info/SC_test_scip_baseline/primal_gap_vs_time.pdf}
        \caption{SC}
        \label{fig:scip_behavior_sc}
    \end{subfigure}

    \begin{subfigure}[t]{\solverpanelwidth\columnwidth}
        \centering
        \includegraphics[width=\linewidth]{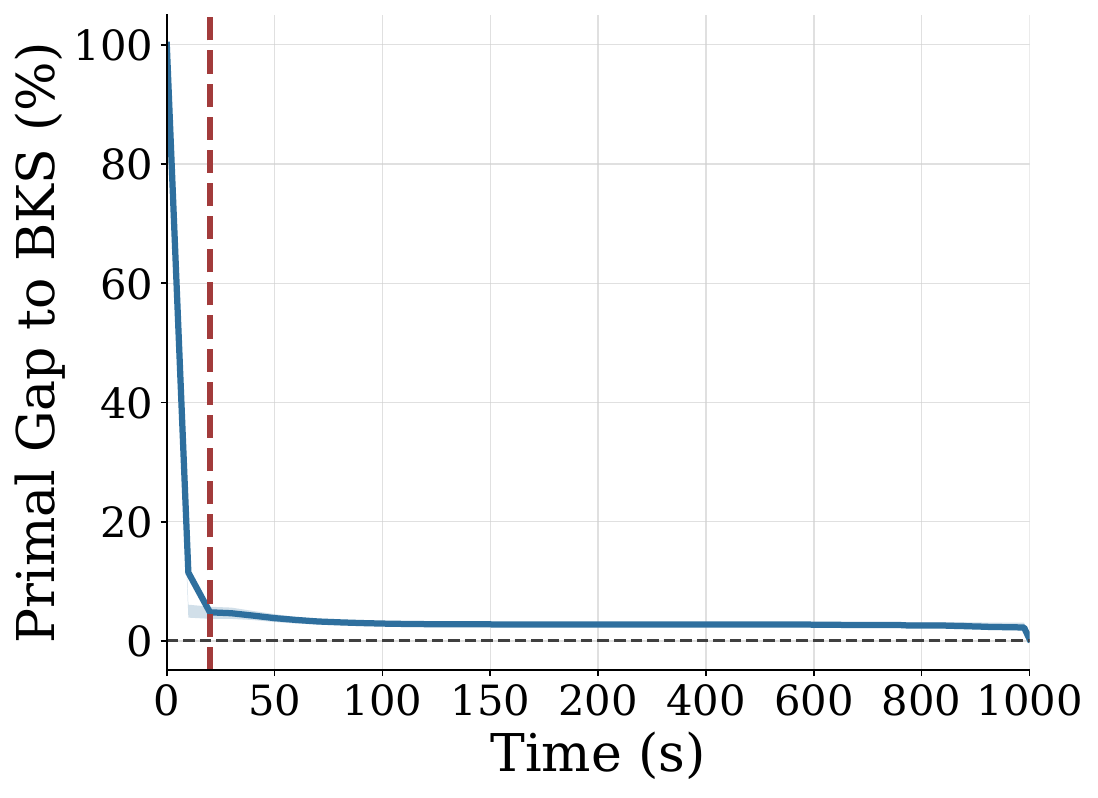}
        \caption{WA}
        \label{fig:scip_behavior_wa}
    \end{subfigure}
    \hfill
    \begin{subfigure}[t]{\solverpanelwidth\columnwidth}
        \centering
        \includegraphics[width=\linewidth]{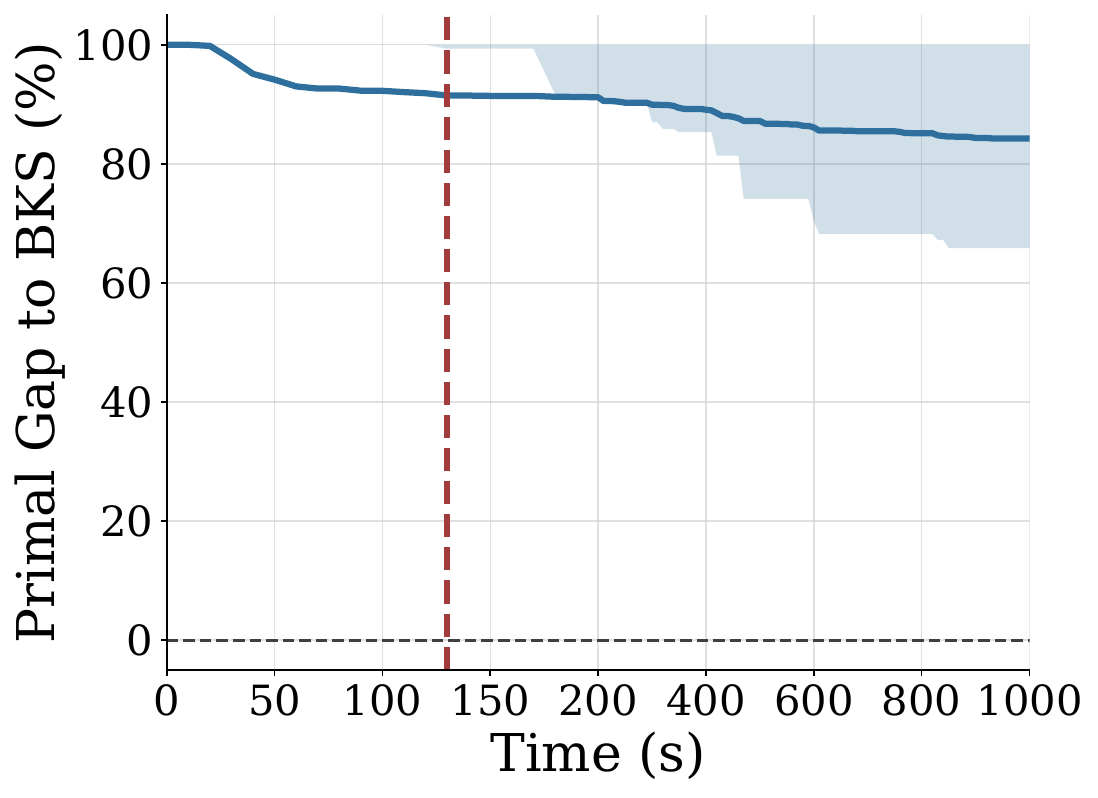}
        \caption{IP}
        \label{fig:scip_behavior_ip}
    \end{subfigure}
    \caption{Primal gap versus solving time for SCIP.}
    \label{fig:scip_solver_behavior}
\end{figure}

\begin{figure}[!ht]
    \centering
    \begin{subfigure}[t]{\solverpanelwidth\columnwidth}
        \centering
        \includegraphics[width=\linewidth]{Figures/early_solution_scip_quality/CA_scip_test_flip_count_vs_best.pdf}
        \caption{CA}
    \end{subfigure}
    \hfill
    \begin{subfigure}[t]{\solverpanelwidth\columnwidth}
        \centering
        \includegraphics[width=\linewidth]{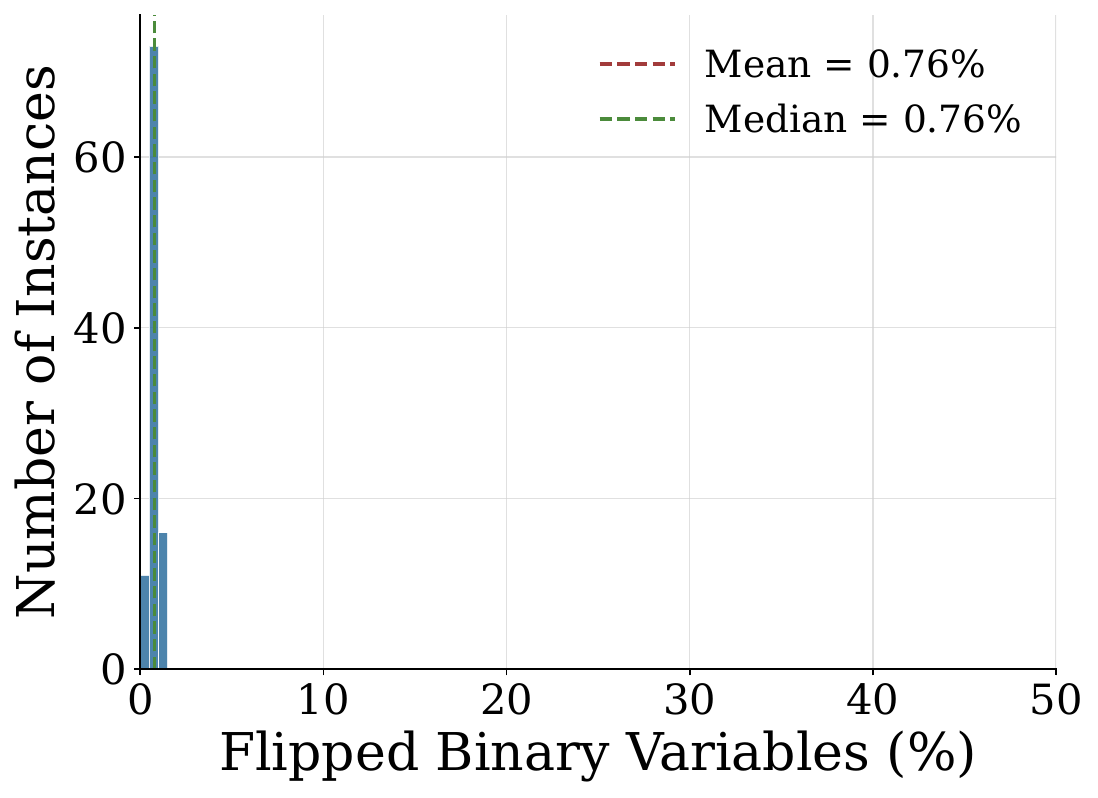}
        \caption{SC}
    \end{subfigure}

    \begin{subfigure}[t]{\solverpanelwidth\columnwidth}
        \centering
        \includegraphics[width=\linewidth]{Figures/early_solution_scip_quality/WA_scip_test_flip_count_vs_best.pdf}
        \caption{WA}
    \end{subfigure}
    \hfill
    \begin{subfigure}[t]{\solverpanelwidth\columnwidth}
        \centering
        \includegraphics[width=\linewidth]{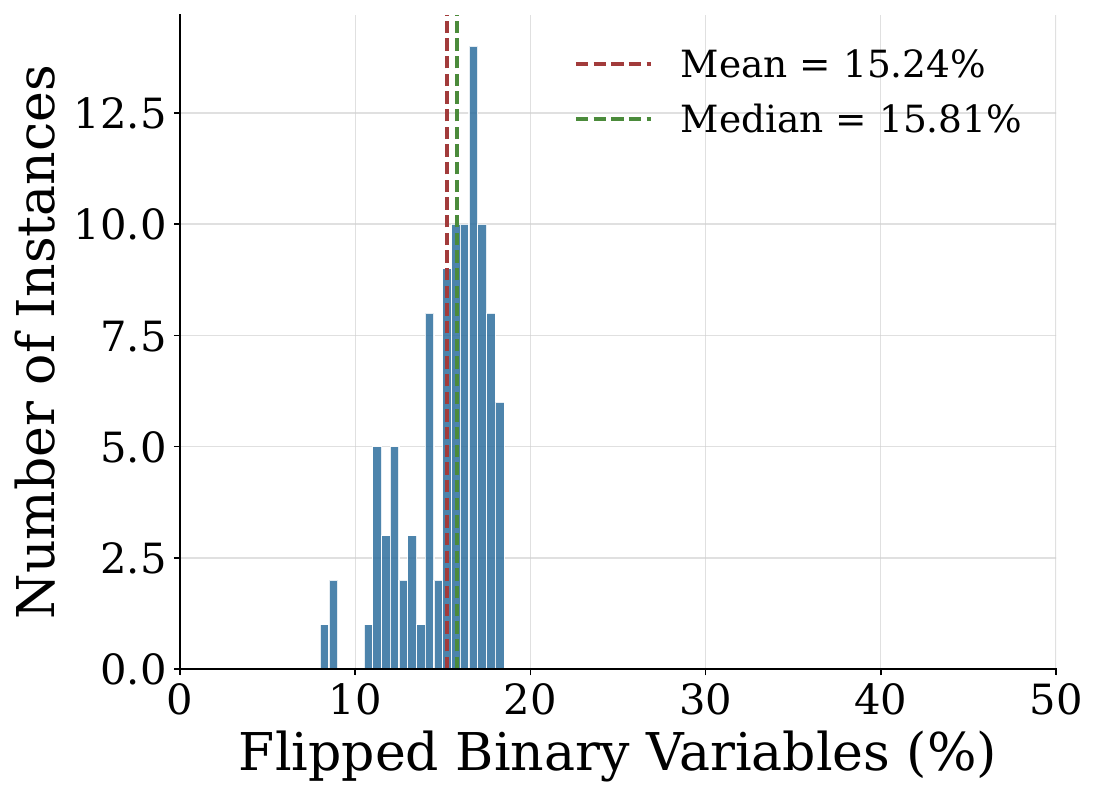}
        \caption{IP}
    \end{subfigure}
    \caption{Early-to-final flip distributions for SCIP; red and green dashed lines denote the mean and median.}
    \label{fig:scip_early_solution_all}
\end{figure}

\begin{figure}[!ht]
    \centering
    \begin{subfigure}[t]{0.40\columnwidth}
        \centering
        \includegraphics[width=\linewidth]{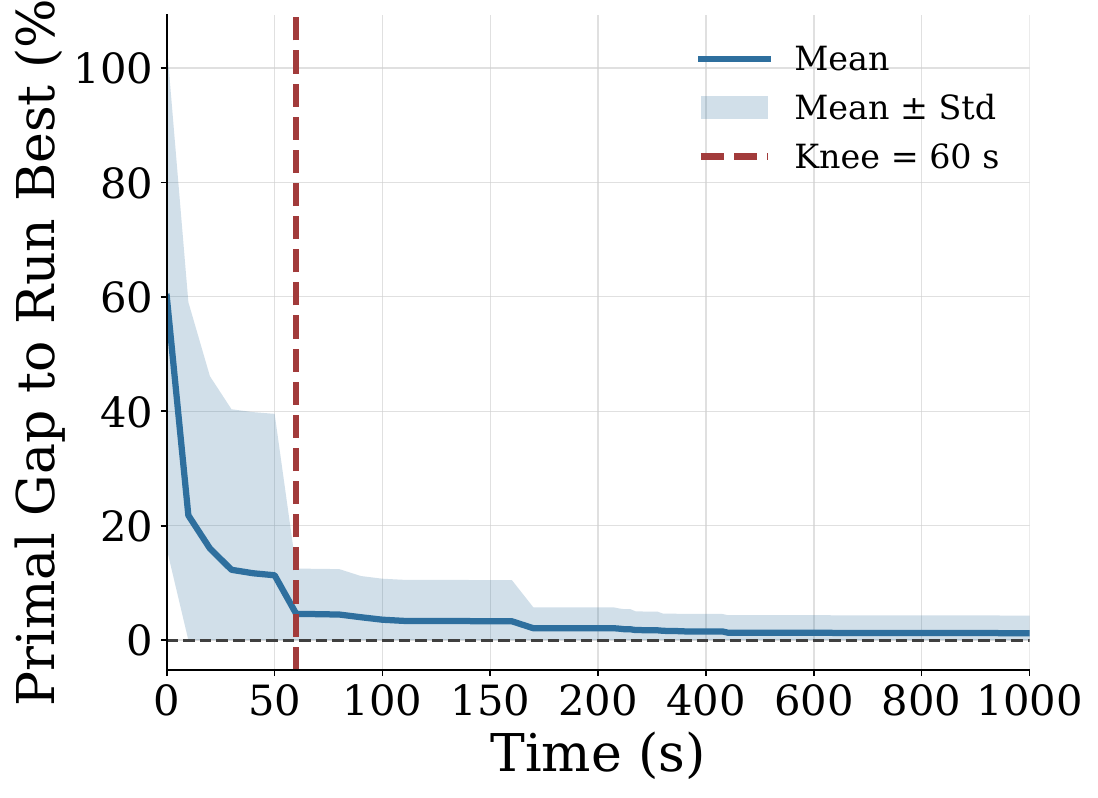}
        \caption{Primal gap over time.}
    \end{subfigure}
    \hfill
    \begin{subfigure}[t]{0.40\columnwidth}
        \centering
        \includegraphics[width=\linewidth]{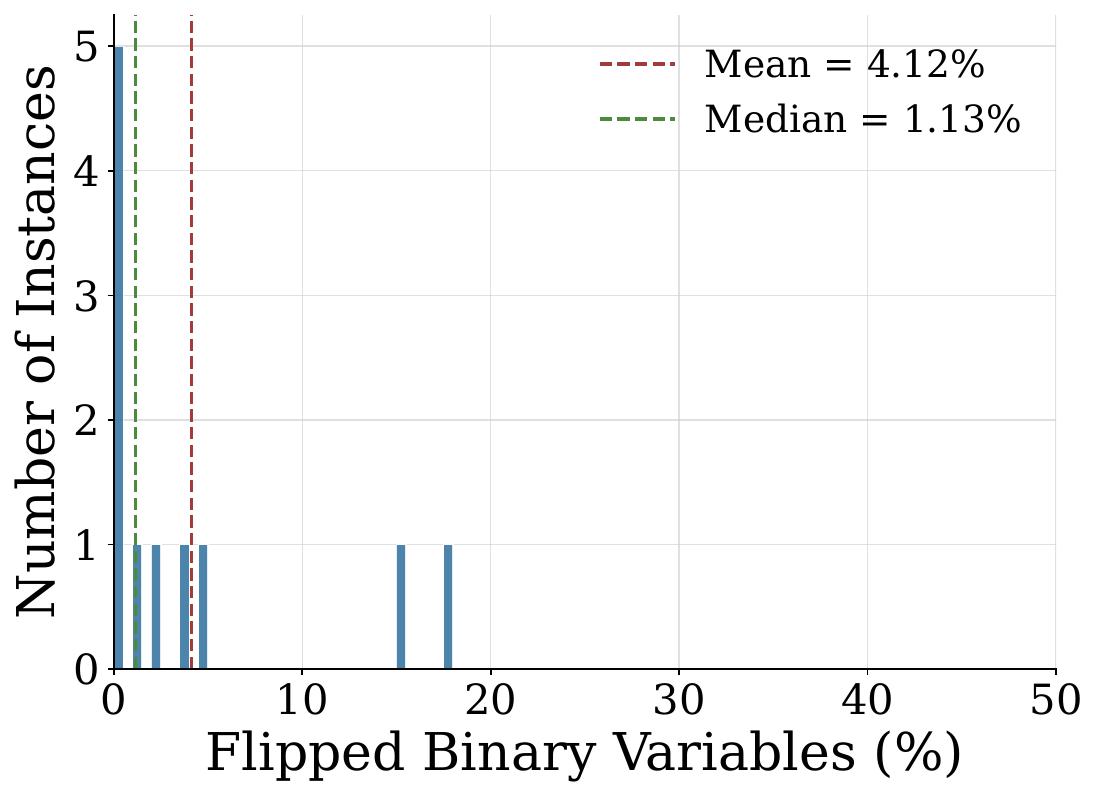}
        \caption{Aggregate flip distribution.}
    \end{subfigure}

    \begin{subfigure}[t]{0.80\columnwidth}
        \centering
        \includegraphics[width=\linewidth]{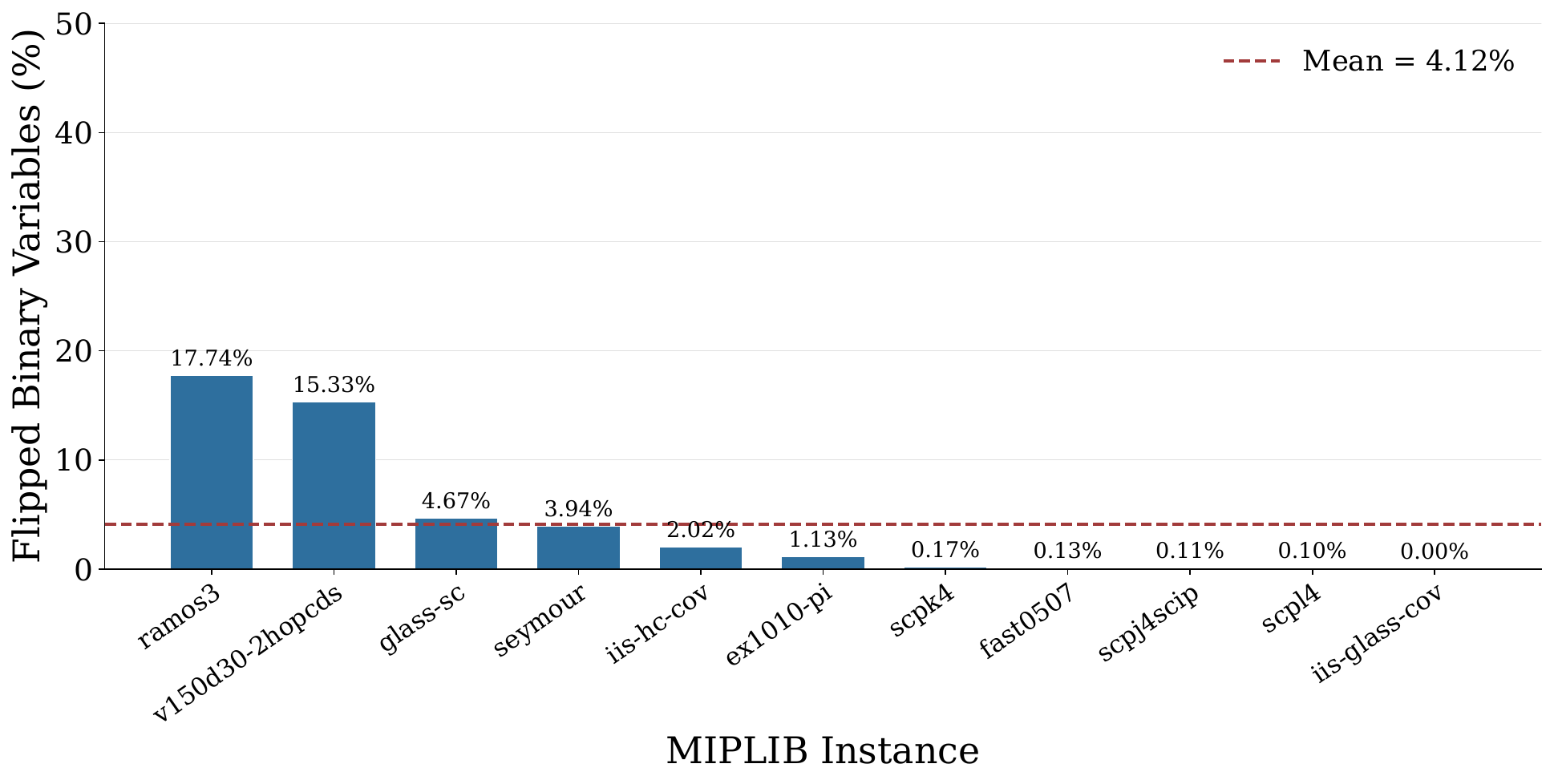}
        \caption{Per-instance flip percentages.}
    \end{subfigure}
    \caption{Solver behavior on the MIPLIB IIS dataset with Gurobi.}
    \label{fig:miplib_early_solution_all}
\end{figure}

\FloatBarrier
\endgroup
\section{Model Structure}
\label{app:model_structure}

\paragraph{Graph Representation.}
We follow the MILP representation of \citet{han2023gnn}, adding only one
dimension to each variable node to encode its early-solution value.
Detailed feature definitions are provided in
Table~\ref{tab:model_input_features}.
\begin{table}[t]
    \centering
    \footnotesize
    \setlength{\tabcolsep}{2.5pt}
    \renewcommand{\arraystretch}{1.05}

    \begin{tabularx}{\columnwidth}{
        @{}
        >{\centering\arraybackslash}p{0.09\columnwidth}
        >{\raggedright\arraybackslash}p{0.29\columnwidth}
        >{\raggedright\arraybackslash}X
        @{}
    }
        \toprule
        \textbf{Index}
        & \textbf{Feature}
        & \textbf{Description} \\
        \midrule

        \multicolumn{3}{@{}l}{\textbf{Variable-node features}} \\
        \addlinespace[1pt]

        0
        & Objective
        & Normalized objective coefficient. \\

        1
        & Variable coefficient
        & Average variable coefficient across all constraints. \\

        2
        & Variable degree
        & Degree of the variable node in the bipartite graph. \\

        3
        & Maximum coefficient
        & Maximum variable coefficient across all constraints. \\

        4
        & Minimum coefficient
        & Minimum variable coefficient across all constraints. \\

        5
        & Variable type
        & Indicator of whether the variable is integer. \\

        6--17
        & Position embedding
        & Binary encoding of the variable's order among all variables. \\

        18
        & Early-solution value
        & Value of the variable in the early solution collected as
        described in Appendix~\ref{app:es_collect}. \\

        \midrule

        \multicolumn{3}{@{}l}{\textbf{Constraint-node features}} \\
        \addlinespace[1pt]

        0
        & Constraint coefficient
        & Average of the nonzero coefficients in the constraint. \\

        1
        & Constraint degree
        & Degree of the constraint node in the bipartite graph. \\

        2
        & Bias
        & Normalized right-hand side of the constraint. \\

        3
        & Sense
        & Sense of the constraint. \\

        \midrule

        \multicolumn{3}{@{}l}{\textbf{Edge features}} \\
        \addlinespace[1pt]

        0
        & Coefficient
        & Coefficient connecting the constraint and variable nodes. \\

        \bottomrule
    \end{tabularx}

    \caption{Graph features for input.}
    \label{tab:model_input_features}
\end{table}

\paragraph{Network architecture.}

We strictly follow the network architecture proposed by
\citet{han2023gnn}, modifying only the variable encoder to accommodate
the additional early-solution feature. Each variable node is represented
by 19 input features, consisting of the original 18 features and the
early-solution value introduced by our method. Constraint nodes and edges
have 4 and 1 input features, respectively. All node representations are
embedded into a 64-dimensional latent space.

For an input feature vector of dimension $d$, the node encoder is defined as
\[
\operatorname{Enc}_{d}(x)
=
\operatorname{ReLU}\!\left(
W_{2}\,
\operatorname{ReLU}\!\left(
W_{1}\operatorname{LN}(x)
\right)
\right),
\]
where $W_{1}\colon\mathbb{R}^{d}\rightarrow\mathbb{R}^{64}$ and
$W_{2}\colon\mathbb{R}^{64}\rightarrow\mathbb{R}^{64}$.
Accordingly, the variable and constraint encoders use
$\operatorname{Enc}_{19}$ and $\operatorname{Enc}_{4}$, respectively,
while the scalar edge feature is normalized by
$\operatorname{LayerNorm}(1)$.

The model performs two rounds of bidirectional message passing between
variable and constraint nodes:
\[
V\rightarrow C\rightarrow V\rightarrow C\rightarrow V.
\]
For a directed edge $j\rightarrow i$, the message is computed as
\[
m_{ij}
=
W_{M}\,
\operatorname{ReLU}\!\left[
\operatorname{LN}\!\left(
W_{L}h_i+W_{E}e_{ij}+W_{R}h_j
\right)
\right],
\]
where $h_i$ and $h_j$ denote the target and source node representations,
and $e_{ij}$ is the encoded edge feature. Incoming messages are summed and
normalized:
\[
\bar m_i
=
\operatorname{LN}\!\left(
\sum_{j\in\mathcal{N}(i)}m_{ij}
\right).
\]
The target node representation is then updated by
\[
h_i'
=
W_{U,2}\,
\operatorname{ReLU}\!\left(
W_{U,1}[h_i\mathbin{\|}\bar m_i]
\right),
\]
where $\mathbin{\|}$ denotes concatenation,
$W_{U,1}\colon\mathbb{R}^{128}\rightarrow\mathbb{R}^{64}$, and
$W_{U,2}\colon\mathbb{R}^{64}\rightarrow\mathbb{R}^{64}$.

Finally, each variable representation is mapped to a scalar logit:
\[
\ell_v
=
w_{\mathrm{out}}^{\top}
\operatorname{ReLU}\!\left(
W_{\mathrm{out}}h_v
\right),
\]
where the final output layer has no bias. The predicted probability
$\sigma(\ell_v)$ indicates whether the early-solution value of variable
$v$ is consistent with its value in the full-budget reference solution.

\FloatBarrier
\section{Early Solution Collection Details}
\label{app:es_collect}
We first quantify the progress of the dual bound using the average
dual-gap decay rate, and then describe the early-solution collection
procedure in Algorithm~\ref{alg:early_solution_collection}. Specifically,
for the $k$-th monitoring interval, the decay rate is defined as
\begin{equation}
    r_k
    =
    \frac{\left|D_{k_a}-D_{k_b}\right|}
         {t_{k_b}-t_{k_a}},
    \label{eq:dual_gap_decay_rate}
\end{equation}
where $k_a$ and $k_b$ denote the observations at the beginning and end
of the interval, respectively; $D_{k_a}$ and $D_{k_b}$ are the
corresponding dual gaps; and $t_{k_a}$ and $t_{k_b}$ are their
wall-clock times. Thus, $r_k$ measures the average absolute reduction
in the dual gap per unit time during the $k$-th monitoring interval.
\begin{algorithm}[ht]
\caption{Early Solution Collection}
\label{alg:early_solution_collection}
\begin{algorithmic}[1]
\REQUIRE MILP instance \(G\), minimum probing time \(T_{\min}\), maximum probing time \(T_{\max}\), window size \(w\), decay threshold \(\epsilon\)
\ENSURE Early solution \(X^{ES}\) and probing time \(t_{\mathrm{probe}}\)
\STATE Initialize an empty callback record list \(\mathcal{R}\).
\STATE Start a probing solver run on \(G\).
\WHILE{the solver is running}
    \IF{a MIPSOL callback occurs at time \(t_k\)}
        \STATE Record the explored solution \(X_k\) and relative dual gap \(D_k\).
        \STATE Append \((t_k,X_k,D_k)\) to \(\mathcal{R}\).
        \IF{\(|\mathcal{R}|\ge w\)}
            \STATE Let \(k_a\) and \(k_b\) be the first and last records in the latest \(w\)-record window.
            \STATE Compute \(r_k\) by Eq.~(\ref{eq:dual_gap_decay_rate}).
            \IF{\(t_k\ge T_{\min}\) and \(r_k<\epsilon\)}
                \STATE Stop the probing run.
            \ENDIF
        \ENDIF
    \ENDIF
    \IF{the elapsed probing time reaches \(T_{\max}\)}
        \STATE Stop the probing run.
    \ENDIF
\ENDWHILE
\STATE Set \(X^{ES}\) to the last explored solution during probing.
\STATE Set \(t_{\mathrm{probe}}\) to the elapsed probing time.
\RETURN \(X^{ES}\) and \(t_{\mathrm{probe}}\).
\end{algorithmic}
\end{algorithm}

\section{Experimental Details}
\label{app:experimental_details}

\subsection{Benchmark Details}

We evaluate on the four benchmark families used by Apollo-MILP
\cite{liu2025apollomilp}. Combinatorial Auctions (CA) instances follow the CATS
generator \cite{leytonbrown2000auction}, Set Covering (SC) follows the classical
construction of \citet{balas1980setcovering}, and Workload Apportionment (WA)
and Item Placement (IP) come from ML4CO \cite{gasse2021ml4co}. CA is a
maximization problem, whereas SC, WA, and IP are minimization problems.
Table~\ref{tab:dataset_statistics} reports average instance sizes.
We use 240 training, 60 validation, and 100 testing instances, following the settings in \citet{han2023gnn} and \citet{liu2025apollomilp}. We train all models with 500 epochs and a learning rate of 0.001.

\begin{table}[t]
  \centering
  
  \setlength{\tabcolsep}{2.5pt}

  \begin{tabular}{lrrrr}
    \toprule
    & CA & SC & IP & WA \\
    \midrule
    Constraint Number & 2590.33 & 3000 & 195 & 64306 \\
    Variable Number & 1500 & 5000 & 1083 & 61000 \\
    Binary Variables Number & 1500 & 5000 & 1050 & 1000 \\
    Continuous Variables Number & 0 & 0 & 33 & 60000 \\
    Integer Variables Number & 0 & 0 & 0 & 0 \\
    \bottomrule
  \end{tabular}
  \caption{Statistical information of the benchmark instances.}
    \label{tab:dataset_statistics}
\end{table}

\subsection{Protocols for Combining EnCore with Downstream Search Methods}
\label{app:downstream_protocol}

This section details how EnCore is integrated with the three downstream
methods evaluated in Section~\ref{sec:experiments}. Following the evaluation
protocol proposed by \citet{han2023gnn} and
\citet{liu2025apollomilp}, every run is given a total wall-clock
budget of $1000$s, and all overhead introduced by EnCore---including the time
spent collecting the early solution---is counted against this budget. Let
$t_{ES}$ denote the time used to obtain the early solution $x^{ES}$ on a given
instance.
\paragraph{Predict-and-Search and Neural Diving.}
Both methods consume the predictor output once, before the final solve, so the
integration requires no change to their search procedures. Given an instance,
we run the solver on the original MILP for $t_{ES}$ seconds and record the
early solution $x^{ES}$ together with the associated solving-process features.
The consistency predictor then assigns each binary variable a score
$\bar{p}_i$, the estimated probability that $x_i^{ES}$ agrees with the
reference solution. Since both methods specify their fixing budgets separately
for the two binary values, selection is value-conditioned: we rank the
variables with $x_i^{ES} = 0$ by $\bar{p}_i$ and retain the top $k_0$, and
independently retain the top $k_1$ among those with $x_i^{ES} = 1$. On CA
under Predict-and-Search, for example, the setting $(k_0, k_1, \Delta) =
(600, 0, 20)$ selects the $600$ zero-valued variables with the highest
consistency scores and no one-valued ones, and the resulting partial
assignment $\{x_i = 0\}$ defines the center of a search neighborhood of radius
$\Delta = 20$. Neural Diving consumes the selection in the same way except
that the assignments are imposed as hard fixings and the reduced MILP is
solved. In both cases the solver is restarted from scratch with a budget of
$1000 - t_{ES}$ seconds, so that the end-to-end wall-clock time matches that
of the baselines.

\paragraph{Apollo-MILP.}
Apollo-MILP performs four prediction--correction rounds with per-round budgets
of $100$s, $100$s, $200$s, and $600$s. Extending the protocol above directly
would require collecting a fresh early solution at the start of every round,
which is too costly: the collection phase alone would consume a
non-negligible fraction of the $100$s rounds. We instead adopt the following
assumption. Since Apollo-MILP realizes its fixing step by imposing bound
constraints on the selected variables rather than by restructuring the
problem, we assume that the early search behavior of the bounded problem stays
close to that of the original instance, so that the early values realized on
the original instance remain informative anchors throughout the iterative
process. Under this assumption, we collect the early solution only once, on
the original instance and in exactly the same way as above, and reuse $x^{ES}$
and its consistency scores in the prediction step of every round. Each round
applies the same value-conditioned selection, restricted to the variables not
yet fixed in earlier rounds. To keep the total budget unchanged, the
collection cost $t_{ES}$ is charged to the final round, whose budget becomes
$600 - t_{ES}$ seconds, while the first three rounds are left untouched; the
overall wall-clock time thus remains $1000$s.

\subsection{Evaluation Protocol}
\label{sec:eval_protocol}

\paragraph{Hardware.}
All experiments are conducted on a Linux server (Ubuntu 20.04.6 LTS) equipped
with an Intel Xeon Platinum 8124M CPU (72 logical cores, 3.50GHz), 256GB of
RAM, and an NVIDIA GeForce RTX 3090 GPU (24GB). Neural network training and
inference run on the GPU, while all solver calls run on the CPU.
The solver configuration (e.g., thread count) is kept identical across all
methods on each benchmark. 

\paragraph{Metrics.}
We report the final feasible objective and its absolute primal gap
\begin{equation}
    \mathrm{gap}_{\mathrm{abs}} = |z - z_{\mathrm{BKS}}|,
\end{equation}
where $z$ is the mean final objective of a method and $z_{\mathrm{BKS}}$ is
the best-known solution (BKS) on the corresponding benchmark. The BKS is
defined as the best mean final objective observed among all evaluated methods,
including a 3,600-second Gurobi run: Gurobi provides this reference on SC, WA,
and IP, while \ours{}-PS provides it on CA. The BKS is thus an in-study
best-observed reference rather than an independently certified optimum. To
quantify the benefit of our method over a downstream baseline, we additionally
report the relative gap reduction
\begin{equation}
    \mathrm{Improvement}
    = \frac{\mathrm{gap}_{\mathrm{base}} - \mathrm{gap}_{\mathrm{ours}}}
           {\mathrm{gap}_{\mathrm{base}}} \times 100\%.
    \label{eq:gap_improvement}
\end{equation}

\subsection{Zero-Shot MIPLIB IIS Transfer Protocol}

The IIS setting traces to \citet{wang2024digmilp}, which formed a two-instance dataset from
MIPLIB: \texttt{iis-glass-cov} was used for training and
\texttt{iis-hc-cov} for testing. Apollo-MILP later
expanded this setting into an eleven-instance subset of MIPLIB 2017
\cite{gleixner2021miplib,liu2025apollomilp}. Specifically, it selected
instances according to similarities measured with 100 human-designed
structural features, then discarded instances whose presolving exceeded 300
seconds or whose inference exceeded GPU memory. Apollo-MILP used eight of the
resulting instances for training and three---\texttt{ramos3},
\texttt{scpj4scip}, and \texttt{scpl4}---for testing. In contrast, we use all
eleven instances exclusively as a zero-shot test set. The instances range from
214 to 200,000 variables and from 32,805 to 2,000,000 nonzero coefficients.

In this experiment, both predictors are trained on the union of the CA, SC, WA, and IP training
sets. They are validated on the union of the corresponding validation sets,
and checkpoint selection uses only this combined validation set. The selected
checkpoints are frozen before evaluation on IIS. No MIPLIB instance is used
for training, validation, fine-tuning, or checkpoint
selection. 

Table~\ref{tab:miplib_iis_average} gives the overall performance while Table~\ref{tab:miplib_iis_instance} gives the instance-level results. The PS
improvement is mainly attributable to \texttt{ramos3} and
\texttt{ex1010-pi}, partially offset by \texttt{scpk4}. Apollo+Ours improves
\texttt{ex1010-pi}, degrades on \texttt{ramos3}, and ties on the remaining
instances. 
\begin{table}[H]
    \centering
    
    \small
    \setlength{\tabcolsep}{5pt}
    \begin{tabular}{@{}llrr@{}}
        \toprule
        Predictor & Downstream & Mean Obj $\downarrow$ & Feas. \\
        \midrule
        Origin & \nd{} & 243.00$^{\dagger}$ & 3/11 \\
        \ours{} & \nd{} & 173.45 & \textbf{11/11} \\
        \midrule
        Origin & \ps{} & 172.00 & 11/11 \\
        \ours{} & \ps{} & \textbf{171.73} & 11/11 \\
        \midrule
        Origin & \apollo{}  & 172.91 & 11/11 \\
        \ours{} & \apollo{} & 172.82 & 11/11 \\
        \bottomrule
    \end{tabular}
    \caption{Zero-shot cross-family transfer to the MIPLIB IIS subset. Mean
    Obj is the mean of the per-instance final objectives. Feas. reports the number of instances with a finite feasible objective. $^{\dagger}$ means computed only over its three feasible
    instances.}
    \label{tab:miplib_iis_average}
\end{table}

\begin{table}[ht]
    \centering
    
    \scriptsize
    \setlength{\tabcolsep}{3.4pt}
    \begin{tabular}{@{}lrrrrrr@{}}
        \toprule
        & \multicolumn{2}{c}{ND} & \multicolumn{2}{c}{PS}
        & \multicolumn{2}{c}{Apollo-MILP} \\
        \cmidrule(lr){2-3}\cmidrule(lr){4-5}\cmidrule(lr){6-7}
        Instance & GCN & Ours & GCN & Ours & GCN & Ours \\
        \midrule
        \texttt{ex1010-pi}       & -- & 242.00 & 237.00 & \textbf{236.00} & 241.00 & 238.00 \\
        \texttt{fast0507}        & -- & \textbf{174.00} & \textbf{174.00} & \textbf{174.00} & \textbf{174.00} & \textbf{174.00} \\
        \texttt{glass-sc}        & -- & \textbf{23.00} & \textbf{23.00} & \textbf{23.00} & \textbf{23.00} & \textbf{23.00} \\
        \texttt{iis-glass-cov}   & -- & \textbf{21.00} & \textbf{21.00} & \textbf{21.00} & \textbf{21.00} & \textbf{21.00} \\
        \texttt{iis-hc-cov}      & -- & \textbf{17.00} & \textbf{17.00} & \textbf{17.00} & \textbf{17.00} & \textbf{17.00} \\
        \texttt{ramos3}          & -- & 235.00 & 229.00 & \textbf{226.00} & 231.00 & 233.00 \\
        \texttt{scpj4scip}       & \textbf{132.00} & \textbf{132.00} & \textbf{132.00} & \textbf{132.00} & \textbf{132.00} & \textbf{132.00} \\
        \texttt{scpk4}           & 328.00 & 330.00 & \textbf{326.00} & 327.00 & 330.00 & 330.00 \\
        \texttt{scpl4}           & \textbf{269.00} & \textbf{269.00} & \textbf{269.00} & \textbf{269.00} & \textbf{269.00} & \textbf{269.00} \\
        \texttt{seymour}         & -- & \textbf{423.00} & \textbf{423.00} & \textbf{423.00} & \textbf{423.00} & \textbf{423.00} \\
        \texttt{v150d30-2hopcds} & -- & 42.00 & \textbf{41.00} & \textbf{41.00} & \textbf{41.00} & \textbf{41.00} \\
        \bottomrule
    \end{tabular}
    \caption{Final objectives on the eleven MIPLIB IIS instances. All instances
    are minimization problems. Bold marks the lowest objective in each row; a
    dash indicates that no finite feasible solution was found.}
    \label{tab:miplib_iis_instance}
\end{table}

\FloatBarrier
\section{Hyperparameter Settings}
\label{app:hyperparameters}

\subsection{Early Solution Collection}
For early-solution collection, we set the sliding-window size to \(w=5\), the
dual-gap decay threshold to \(\epsilon=0.01\%\)/s, the minimum probing time to
\(T_{\min}=20\) seconds, and the maximum probing time to \(T_{\max}=60\)
seconds. We use an ensemble size of \(K=3\) for SC, CA and WA and use \(K=2\) for IP. 

\subsection{Neighborhood Search Parameters}
We report the hyper-parameters we used in our study. \ours{} and baselines use the same hyper-parameters. Table~\ref{tab:partial_solution_parameters} shows the hyperparameters of \ps{} and \ours{}-PS, and the \nd{} and \ours{}-ND uses the same $(k_0, k_1)$. Table~\ref{tab:apollo_hyperparameters} shows the Apollo hyperparameters for baseline \apollo{} and \ours{}-apollo.
\begin{table}[htbp]
    \centering
    \setlength{\tabcolsep}{1pt}
    \small
    \begin{tabular}{lcccc}
        \toprule
        Benchmark & CA & SC & IP & WA \\
        \midrule
        PS+Gurobi & $(600,0,20)$ & $(2000,0,100)$ & $(400,5,10)$ & $(0,500,10)$ \\
        PS+SCIP   & $(400,0,20)$ & $(2000,0,100)$ & $(400,5,1)$  & $(0,600,5)$ \\
        \bottomrule
    \end{tabular}
    \caption{The partial solution size parameters $(k_0, k_1)$ and neighborhood parameter $\Delta$.}
    \label{tab:partial_solution_parameters}
\end{table}
\begin{table}[htbp]
    \centering
    \setlength{\tabcolsep}{1pt}
    \small
    \begin{tabular}{lcccc}
        \toprule
        & CA & SC & IP & WA \\
        \midrule
        Iteration 1
        & $(400,0,60)$
        & $(1000,0,200)$
        & $(100,20,50)$
        & $(20,200,100)$ \\
        
        Iteration 2
        & $(200,0,30)$
        & $(500,0,100)$
        & $(40,15,20)$
        & $(10,100,50)$ \\
        
        Iteration 3
        & $(100,0,15)$
        & $(250,0,50)$
        & $(20,15,10)$
        & $(10,5,5)$ \\
        
        Iteration 4
        & $(50,0,10)$
        & $(10,0,5)$
        & $(5,50,30)$
        & $(1,10,5)$ \\
        \bottomrule
    \end{tabular}
    \caption{Hyperparameters $(k_0^{(i)}, k_1^{(i)}, \Delta^{(i)})$ for Apollo-MILP.}
    \label{tab:apollo_hyperparameters}
\end{table}

\FloatBarrier
\section{Details and Proofs for the Theoretical Analysis}
\label{app:theoretical_proofs}

\subsection{Population Setting and Notation}

Let \(\mathcal D\) be the probability distribution over all triples
consisting of an MILP graph, the early solution produced by the fixed
collection procedure, and the full-budget solution.  Draw
\[
    Z=(G,X^{ES},X^*)\sim\mathcal D.
\]
Let \(\mathcal B(G)\) be the set of binary-variable indices in \(G\).
Assume \(|\mathcal B(G)|\ge1\) almost surely.
To represent the per-instance variable average compactly, conditional on
\(Z\), choose \(J\) from \(\mathcal B(G)\), giving every index probability
\(1/|\mathcal B(G)|\).  We use \(J\) for this random index in population
quantities and \(r\) for a deterministic summation index.

Fix the message-passing depth \(L\).  For \(r\in\mathcal B(G)\), let
\(\mathcal R_r\) denote the part of the static graph and its features within
\(L\) message-passing steps of variable \(r\), and let
\(X^{ES}_{\mathcal R_r}\) denote the early assignments attached to the
variable nodes in this local input.

For any assignment predictor \(h\), define
\begin{equation*}
\begin{aligned}
    \ell_h(Z)
    &:=
    \frac{1}{|\mathcal B(G)|}
    \sum_{r\in\mathcal B(G)}
    \mathbf 1\{h_r\ne x_r^*\},\\
    R(h)
    &:=
    \mathrm E[\ell_h(Z)]
    =
    \Pr(h_J\ne x_J^*),\\
    \operatorname{Acc}(h)
    &:=
    1-R(h).
\end{aligned}
\end{equation*}
Here a direct prediction \(h_r\) is a function of \(\mathcal R_r\), whereas
an induced consistency prediction may also use
\(X^{ES}_{\mathcal R_r}\).  The corresponding posteriors are
\[
\begin{aligned}
    \eta_G&:=\Pr(x_J^*=1\mid \mathcal R_J),\\
    \eta_{ES}&:=\Pr(x_J^*=1\mid
        \mathcal R_J,X^{ES}_{\mathcal R_J}).
\end{aligned}
\]
Their Bayes errors are
\begin{equation*}
\begin{aligned}
    b_G
    &:=
    \mathrm E[\min\{\eta_G,1-\eta_G\}],\\
    b_{ES}
    &:=
    \mathrm E[\min\{\eta_{ES},1-\eta_{ES}\}].
\end{aligned}
\end{equation*}
Thus \(A_{\mathrm{sol}}^*=1-b_G\) is the best population accuracy attainable
from the static local input, and \(1-b_{ES}\) is the best assignment
accuracy attainable when the early assignments are also observed.

\subsection{Consistency--Assignment Equivalence}

\paragraph{Lemma 1 (label--assignment equivalence).}
For \(r\in\mathcal B(G)\), let
\(y_r=\mathbf 1\{x_r^{ES}=x_r^*\}\).  Given a binary consistency prediction
\(\widehat y_r\), retain \(x_r^{ES}\) when \(\widehat y_r=1\) and use its
binary complement otherwise.  The induced assignment \(\widehat x_r\)
satisfies
\begin{equation*}
    \mathbf 1\{\widehat x_r\ne x_r^*\}
    =
    \mathbf 1\{\widehat y_r\ne y_r\}.
\end{equation*}
Equivalently, for a consistency rule \(a\), define
\[
    h_a(G,X^{ES})_r
    :=
    x_r^{ES}\mathbin{\oplus}(1-a(G,X^{ES})_r).
\]
Then
\begin{equation}
    \ell_{h_a}(Z)
    =
    \frac{1}{|\mathcal B(G)|}
    \sum_{r\in\mathcal B(G)}
    \mathbf 1\{a(G,X^{ES})_r\ne y_r\}.
    \label{eq:app_xor_identity}
\end{equation}

\paragraph{Proof.}
For binary values,
\(x_r^*=x_r^{ES}\mathbin{\oplus}(1-y_r)\).  XOR by the same value preserves
equality, so \(h_a(G,X^{ES})_r\ne x_r^*\) exactly when
\(a(G,X^{ES})_r\ne y_r\).  Averaging over \(\mathcal B(G)\) proves
Eq.~(\ref{eq:app_xor_identity}). \(\Box\)

The induced assignment \(h_a\) is used only to compare variable-level
prediction errors.  In the downstream search method, assignments predicted
inconsistent are left free rather than forcibly complemented, so these
identities characterize the variable-level prediction task, while the
downstream objective and runtime effects are evaluated empirically.

Conversely, every direct solution predictor \(h\) defines a consistency rule
\[
    a_h(G,X^{ES})_r
    :=
    \mathbf 1\{h(\mathcal R_r)=x_r^{ES}\}.
\]
The induced predictor satisfies
\(h_{a_h}(G,X^{ES})_r=h(\mathcal R_r)\).
Thus consistency prediction with the augmented input can reproduce every
direct predictor based on the static local input.

Let \(Y_J=\mathbf 1\{x_J^{ES}=x_J^*\}\) and
\[
    \xi_J
    :=
    \Pr(Y_J=1\mid\mathcal R_J,X^{ES}_{\mathcal R_J}).
\]
For a fixed variable \(r\), \(\xi_r\) denotes the corresponding conditional
posterior evaluated at \(r\).
Since \(x_J^{ES}\) is observed in the augmented local input,
\begin{equation*}
    \xi_J
    =
    x_J^{ES}\eta_{ES}
    +(1-x_J^{ES})(1-\eta_{ES}).
\end{equation*}
Consequently,
\[
    \min\{\xi_J,1-\xi_J\}
    =
    \min\{\eta_{ES},1-\eta_{ES}\},
\]
so the Bayes consistency error equals \(b_{ES}\).  Under variable-level
\(0\)--\(1\) prediction loss, the optimal rule predicts \(1\) when
\(\xi_J\ge1/2\) and \(0\) otherwise.  By Lemma~1, its induced assignment
accuracy is \(A_{\mathrm{con}}^*=1-b_{ES}\).
Because the label transformation is bijective after \(x_J^{ES}\) is observed,
this is also the Bayes accuracy of a direct assignment predictor using the
same augmented input.  Theorem~\ref{thm:info-gain} therefore compares the
information in the augmented input with that in the static input, rather than
attributing a Bayes advantage to the consistency relabeling itself.

\subsection{Proof of Theorem~\ref{thm:info-gain}}

The population accuracy difference can be written as
\begin{equation}
    A_{\mathrm{con}}^*-A_{\mathrm{sol}}^*
    =
    b_G-b_{ES}
    \ge0.
    \label{eq:app_bayes_advantage}
\end{equation}
The inequality is strict if there exists an event
\(\mathcal E\in\sigma(\mathcal R_J)\) with \(\Pr(\mathcal E)>0\) such that,
almost surely on \(\mathcal E\),
\[
\begin{aligned}
    \Pr(\eta_{ES}<\tfrac12\mid\mathcal R_J)&>0,\\
    \Pr(\eta_{ES}>\tfrac12\mid\mathcal R_J)&>0.
\end{aligned}
\]
In words, the same static local input can be paired with early assignments
that make either final value more likely.

\paragraph{Proof of Theorem~\ref{thm:info-gain}.}
The tower property gives
\begin{equation}
    \eta_G
    =
    \mathrm E[\eta_{ES}\mid\mathcal R_J].
    \label{eq:app_posterior_tower}
\end{equation}
Using
\(\min\{t,1-t\}=\tfrac12-|t-\tfrac12|\), the Bayes accuracy gap is
\begin{equation*}
\begin{aligned}
    A_{\mathrm{con}}^*-A_{\mathrm{sol}}^*
    &=
    \mathrm E[|\eta_{ES}-\tfrac12|]\\
    &\quad-
    \mathrm E[|\eta_G-\tfrac12|].
\end{aligned}
\end{equation*}
Conditional Jensen's inequality and
Eq.~(\ref{eq:app_posterior_tower}) imply
\[
    \mathrm E[|\eta_{ES}-\tfrac12|\mid\mathcal R_J]
    \ge
    |\eta_G-\tfrac12|.
\]
Taking expectations proves Eq.~(\ref{eq:app_bayes_advantage}).  If the
refined posterior lies on both sides of \(1/2\) with positive conditional
probability, the conditional absolute-value inequality is strict.  The
stated positive-probability condition therefore gives
\(A_{\mathrm{con}}^*>A_{\mathrm{sol}}^*\). \(\Box\)

\paragraph{Scope of the comparison.}
The two optima compared in Theorem~\ref{thm:info-gain} range over functions
of the same receptive field: \(A_{\mathrm{sol}}^*\) is the Bayes accuracy
given the static local input \(\mathcal R_J\), and \(A_{\mathrm{con}}^*\) is
the Bayes accuracy given the augmented local input
\((\mathcal R_J,X^{ES}_{\mathcal R_J})\), at the same fixed depth \(L\).
Statements in the main text about ``any solution predictor'', or about what
``no predictor on instance features alone'' can achieve, therefore quantify
over predictors whose decision at a variable is a measurable function of
its static local input, however large that function class is; they do not
compare against predictors with a strictly larger input scope.  In
particular, if \(X^*\) were almost surely determined by the full graph
\(G\)---for example, a unique optimal solution returned by a deterministic
solver---then the whole-graph posterior would be \(\{0,1\}\)-valued,
posterior crossing could not occur at the whole-graph level, and a
whole-graph predictor would attain perfect accuracy in principle, leaving
no information-theoretic gain for the early solution.  The strict-gain
content of Theorem~\ref{thm:info-gain} is thus per locality level: for
every fixed message-passing depth, augmenting the local input strictly
improves the best achievable accuracy whenever posterior crossing occurs at
that depth.  This is the operative comparison for the models considered in
this line of work, including ours and the baselines, which are all
message-passing GNNs whose receptive field is fixed by the architecture and
is typically far smaller than the benchmark graphs.  Moreover, on practical
benchmarks \(X^*\) retains genuine randomness given \(G\): instances
frequently admit multiple optimal or near-optimal solutions, and the
incumbent returned under a time budget depends on tie-breaking, thread
timing, and budget truncation rather than on \(G\) alone.  Under such
residual randomness the whole-graph posterior is not \(\{0,1\}\)-valued
either, so posterior crossing---and hence a strict gain from the early
solution---can occur at any input scope, including the full graph.

\subsection{Sparse-Correction Identity}

The Bayes consistency rule is
\[
    a^*(G,X^{ES})_r
    :=
    \mathbf 1\{\xi_r\ge\tfrac12\}.
\]
Define the event
\[
    \mathcal C^*
    :=
    \{\xi_J<\tfrac12\}
    =
    \{a^*(G,X^{ES})_J=0\},
\]
and define
\[
    p:=\Pr(Y_J=0),
    \qquad
    \rho:=\Pr(\mathcal C^*).
\]
If \(\rho>0\), also define
\[
    q:=\Pr(Y_J=0\mid\mathcal C^*).
\]
Then \(q>1/2\), and \(a^*\) satisfies
\begin{equation}
\begin{aligned}
    b_{ES}
    =
    R(h_{a^*})
    &=
    p-\rho(2q-1),\\
    \Delta_{\mathrm B}
    :=
    A_{\mathrm{con}}^*-A_{\mathrm{sol}}^*
    &=
    b_G-p+\rho(2q-1).
\end{aligned}
    \label{eq:app_consistency_advantage}
\end{equation}
Thus \(\Delta_{\mathrm B}>0\) exactly when
\begin{equation}
    p<b_G+\rho(2q-1).
    \label{eq:app_population_condition}
\end{equation}
For \(p\ge b_G\) and \(\rho>0\), this is equivalent to
\begin{equation*}
    q>
    \frac12+\frac{p-b_G}{2\rho}.
\end{equation*}
If \(\rho=0\), then \(b_{ES}=p\) and
\(\Delta_{\mathrm B}=b_G-p\).

\paragraph{Derivation.}
For any consistency rule \(a\), retaining an assignment has error
\(1-Y_J\), whereas using its binary complement has error \(Y_J\).  Hence
\begin{equation}
    R(h_a)
    =
    p+
    \mathrm E[(2Y_J-1)(1-a(G,X^{ES})_J)].
    \label{eq:app_consistency_risk_identity}
\end{equation}
On \(\mathcal C^*\), the conditional probability of \(Y_J=0\) is \(q\).
Substituting \(a^*\) into Eq.~(\ref{eq:app_consistency_risk_identity}) gives
\[
    R(h_{a^*})
    =
    p+\rho(1-q)-\rho q
    =
    p-\rho(2q-1).
\]
Moreover,
\[
    q
    =
    1-\mathrm E[\xi_J\mid\mathcal C^*]
    >
    \tfrac12.
\]
The rule \(a^*\) is the pointwise Bayes classifier for \(Y_J\), so
Lemma~1 gives \(R(h_{a^*})=b_{ES}\).  Subtracting this risk from \(b_G\)
proves Eq.~(\ref{eq:app_consistency_advantage}); the remaining statements
follow by rearrangement.  If \(\rho=0\), \(a^*\) retains the early solution
almost surely, and its error is \(p\). \(\Box\)

\subsection{Finite-Sample Analysis}

The population result ranges over all possible rules.  For the finite-sample
result, fix before observing the training instances a finite class
\(\mathcal A\) of \(N:=|\mathcal A|\ge1\) consistency rules whose
variable-level decisions use the augmented local inputs defined above.  Let
\(Z_1,\ldots,Z_m\) be i.i.d.\ training instances from \(\mathcal D\).
Suppose training and test instances follow \(\mathcal D\).  Define
\begin{equation*}
\begin{aligned}
    A_{\mathcal A}^*
    &:=
    \max_{a\in\mathcal A}\operatorname{Acc}(h_a),\\
    \Delta_{\mathcal A}
    &:=
    A_{\mathcal A}^*-A_{\mathrm{sol}}^*
    =
    b_G-\min_{a\in\mathcal A}R(h_a).
\end{aligned}
\end{equation*}
Recall
\(\Delta_{\mathrm B}:=A_{\mathrm{con}}^*-A_{\mathrm{sol}}^*\), and define the
finite-class approximation gap
\[
    \alpha_{\mathcal A}
    :=
    A_{\mathrm{con}}^*-A_{\mathcal A}^*.
\]
Then
\(\Delta_{\mathcal A}=\Delta_{\mathrm B}-\alpha_{\mathcal A}\), separating the
population information gain from the approximation cost of the finite class.
For an instance \(Z\), let \(s_a(Z)\) be the proportion of binary variables
that rule \(a\) would flip.  Assume every candidate flips at most a proportion
\(\bar s\in[0,1]\) on average:
\begin{equation*}
\begin{aligned}
    s_a(Z)
    &:=
    \frac{1}{|\mathcal B(G)|}
    \sum_{r\in\mathcal B(G)}(1-a(G,X^{ES})_r),\\
    \max_{a\in\mathcal A}\mathrm E[s_a(Z)]
    &\le\bar s.
\end{aligned}
\end{equation*}
Thus \(\Delta_{\mathcal A}\) is the population accuracy difference between
the best rule available in \(\mathcal A\) and the best predictor based only
on the static local input.  It can be nonpositive if the finite class
does not contain a sufficiently accurate consistency rule.

For a predictor that depends on the realized training sample and any training
randomness, \(R(h)\) and \(\operatorname{Acc}(h)\) below denote conditional
population quantities evaluated on a fresh test triple independent of that
training information.

Let \(a_{\mathrm{keep}}(G,X^{ES})_r\equiv1\).  Then
\(h_{a_{\mathrm{keep}}}(G,X^{ES})=X^{ES}\) and
\(R(h_{a_{\mathrm{keep}}})=p\).  Define the loss change relative to this
baseline by
\[
\begin{aligned}
    D_a(Z)
    &:=
    \ell_{h_a}(Z)-\ell_{h_{a_{\mathrm{keep}}}}(Z),\\
    \widehat D_a
    &:=
    \frac1m\sum_{k=1}^mD_a(Z_k).
\end{aligned}
\]

\paragraph{Lemma 2 (uniform concentration).}
Following the finite-class Bernstein--union-bound argument
\citep{boucheron2013concentration}, we apply concentration to instance-level
loss differences relative to the always-retain rule.
For every \(\delta\in(0,1)\), with probability at least \(1-\delta\),
\begin{equation*}
\begin{aligned}
    \sup_{a\in\mathcal A}
    |\widehat D_a-\mathrm E[D_a]|
    &\le
    \beta_{ES},\\
    \beta_{ES}
    &:=
    \sqrt{
        \frac{2\bar s}{m}\ln\frac{2N}{\delta}
    }
    +
    \frac{4}{3m}\ln\frac{2N}{\delta}.
\end{aligned}
\end{equation*}

\paragraph{Proof.}
A rule differs from the always-retain baseline only on assignments that it
marks inconsistent, so
\[
\begin{aligned}
    |D_a(Z)|&\le s_a(Z)\le1,\\
    \operatorname{Var}(D_a(Z))
    &\le\mathrm E[D_a(Z)^2]
    \le\mathrm E[s_a(Z)]
    \le\bar s.
\end{aligned}
\]
Thus, when candidate corrections modify only a small proportion of early
assignments, their loss differences have lower variance; this is the
finite-class sparse-correction effect captured by \(\bar s\).
Also, \(|D_a(Z)-\mathrm E[D_a]|\le2\).  (The sharper bound
\(|D_a(Z)-\mathrm E[D_a]|\le s_a(Z)+\bar s\le1+\bar s\) is available, since
\(|D_a(Z)|\le s_a(Z)\) and \(|\mathrm E[D_a]|\le\bar s\); we retain the
simpler constant \(2\), which only inflates the second-order term and keeps
the penalty identical to the main-text statement.)  To make the additive
constant explicit, start from the two-sided Bernstein inequality in
exponential form
\citep[][Theorem~2.10]{boucheron2013concentration}: for a variance proxy
\(\sigma^{2}\) and centered range \(c\),
\[
    \Pr\!\left(
        |\widehat D_a-\mathrm E[D_a]|\ge\epsilon
    \right)
    \le
    2\exp\!\left(
        -\frac{m\epsilon^{2}}{2\sigma^{2}+2c\epsilon/3}
    \right).
\]
Bounding the right-hand side by \(2e^{-t}\) gives the quadratic condition
\(m\epsilon^{2}\ge2\sigma^{2}t+2ct\epsilon/3\), whose positive root
satisfies
\[
    \epsilon
    =
    \frac{ct}{3m}
    +\sqrt{\frac{2\sigma^{2}t}{m}+\frac{c^{2}t^{2}}{9m^{2}}}
    \le
    \sqrt{\frac{2\sigma^{2}t}{m}}+\frac{2ct}{3m},
\]
where the inequality uses \(\sqrt{a+b}\le\sqrt a+\sqrt b\).  Substituting
the variance proxy \(\sigma^{2}=\bar s\) and the range \(c=2\) shows that
the deviation \(\sqrt{2\bar s t/m}+4t/(3m)\) is sufficient; the additive
form is conservative relative to the exponential form because of this
square-root relaxation.  That is, for every fixed \(a\) and \(t>0\),
\[
    \Pr\left(
        |\widehat D_a-\mathrm E[D_a]|
        >
        \sqrt{\frac{2\bar s t}{m}}+\frac{4t}{3m}
    \right)
    \le2e^{-t}.
\]
Setting \(t=\ln(2N/\delta)\) and applying a union bound over
\(\mathcal A\) proves the claim. \(\Box\)

\paragraph{Theorem~\ref{thm:finite-sample} (restated).}
Let \(\widehat a\) minimize the empirical \(0\)--\(1\) risk over
\(\mathcal A\), and write \(\widehat h_{ES}:=h_{\widehat a}\).  Let
\(\widehat h_{\mathrm{sol}}\) be any possibly data-dependent direct predictor
determined without access to the fresh test triple and whose test-time
decisions use only the static local inputs.  Define
\begin{equation*}
    \varepsilon_m(\delta)
    :=
    2\sqrt{
        \frac{2\bar s}{m}\ln\frac{2N}{\delta}
    }
    +
    \frac{8}{3m}\ln\frac{2N}{\delta}.
\end{equation*}
With probability at least \(1-\delta\),
\begin{equation*}
    \operatorname{Acc}(\widehat h_{ES})
    -
    \operatorname{Acc}(\widehat h_{\mathrm{sol}})
    \ge
    \Delta_{\mathcal A}-\varepsilon_m(\delta)
    =
    \Delta_{\mathrm B}-\alpha_{\mathcal A}
    -\varepsilon_m(\delta).
\end{equation*}
On the same event,
\[
    R(\widehat h_{ES})
    -
    \min_{a\in\mathcal A}R(h_a)
    \le\varepsilon_m(\delta).
\]

\paragraph{Proof.}
Subtracting the same empirical loss of \(a_{\mathrm{keep}}\) from every
candidate preserves the empirical minimizer.  On the event in Lemma~2,
choose
\[
    a_{\mathcal A}^*
    \in
    \operatorname*{arg\,min}_{a\in\mathcal A}R(h_a).
\]
Then
\[
\begin{aligned}
    \mathrm E[D_{\widehat a}]
    &\le
    \widehat D_{\widehat a}+\beta_{ES}
    \le
    \widehat D_{a_{\mathcal A}^*}+\beta_{ES}\\
    &\le
    \mathrm E[D_{a_{\mathcal A}^*}]+2\beta_{ES}.
\end{aligned}
\]
Because \(R(h_a)=p+\mathrm E[D_a]\), it follows that
\begin{equation}
\begin{aligned}
    R(\widehat h_{ES})
    &\le
    R(h_{a_{\mathcal A}^*})+2\beta_{ES}\\
    &=1-A_{\mathcal A}^*+2\beta_{ES}\\
    &=b_G-\Delta_{\mathcal A}+2\beta_{ES}.
\end{aligned}
    \label{eq:app_learned_consistency_risk}
\end{equation}
Condition on the realized training sample.  Any resulting direct predictor
whose test decision uses only the static local input has, by Bayes
optimality,
\begin{equation}
    R(\widehat h_{\mathrm{sol}})\ge b_G.
    \label{eq:app_static_bayes_lower_bound}
\end{equation}
Combining Eqs.~(\ref{eq:app_learned_consistency_risk}) and
(\ref{eq:app_static_bayes_lower_bound}) gives
\[
\begin{aligned}
    \operatorname{Acc}(\widehat h_{ES})
    -
    \operatorname{Acc}(\widehat h_{\mathrm{sol}})
    &=
    R(\widehat h_{\mathrm{sol}})-R(\widehat h_{ES})\\
    &\ge
    \Delta_{\mathcal A}-2\beta_{ES}.
\end{aligned}
\]
Since \(\varepsilon_m(\delta)=2\beta_{ES}\), the theorem follows. \(\Box\)

\paragraph{A sufficient sample size.}
If \(\Delta_{\mathcal A}>0\) and
\begin{equation}
    m>
    \max\left\{
        \frac{32\bar s\ln(2N/\delta)}{\Delta_{\mathcal A}^2},
        \frac{16\ln(2N/\delta)}{3\Delta_{\mathcal A}}
    \right\},
    \label{eq:app_theory_sample_complexity}
\end{equation}
then \(\varepsilon_m(\delta)<\Delta_{\mathcal A}\).  Consequently, with probability at
least \(1-\delta\),
\(\operatorname{Acc}(\widehat h_{ES})
>\operatorname{Acc}(\widehat h_{\mathrm{sol}})\).

\paragraph{Proof.}
Let \(\Lambda:=\ln(2N/\delta)\).  The two bounds in
Eq.~(\ref{eq:app_theory_sample_complexity}) imply
\[
    2\sqrt{\frac{2\bar s\Lambda}{m}}
    <
    \frac{\Delta_{\mathcal A}}{2},
    \qquad
    \frac{8\Lambda}{3m}
    <
    \frac{\Delta_{\mathcal A}}{2}.
\]
Adding the inequalities gives
\(\varepsilon_m(\delta)<\Delta_{\mathcal A}\);
Theorem~\ref{thm:finite-sample} then yields the claim.
\(\Box\)

\subsection{Illustrative Calculations}
\label{app:theory_example}

\paragraph{Population gain.}
The sparse-correction condition can hold even when the early solution itself
is less accurate than direct prediction.  For example, take
\[
    p=0.21,\qquad
    A_{\mathrm{sol}}^*=0.82,\qquad
    \rho=0.12,\qquad
    q=0.70.
\]
These numerical values are purely illustrative.  The right-hand side of
Eq.~(\ref{eq:app_population_condition}) is
\[
    1-A_{\mathrm{sol}}^*+\rho(2q-1)
    =
    0.18+0.12(0.40)
    =
    0.228.
\]
Thus \(p=0.21<0.228\).  The early solution alone has accuracy \(0.79\),
which is below the direct-prediction accuracy \(0.82\), whereas identifying
the inconsistent assignments gives
\[
    A_{\mathrm{con}}^*
    =
    1-\left[0.21-0.12(0.40)\right]
    =
    0.838.
\]
The early-solution input therefore raises the accuracy from \(0.82\) to
\(0.838\), an improvement of \(1.8\) percentage points.

\paragraph{Finite-sample bound.}
For a numerical evaluation of the bound, consider a fixed class of \(N=4\)
correction rules constructed to satisfy \(\mathrm E[s_a]\le0.0180\).  The value
\(\bar s=0.0180\) uses the \(1.80\%\) median WA early-to-final difference in
Figure~\ref{fig:early_solution_quality} as an illustrative correction-budget
scale, with \(\mathrm E[s_a]\le\bar s\) retained as the stated rule-class
assumption.  Set the confidence level to \(1-\delta=0.9\) and take a
deliberately conservative illustrative class advantage
\(\Delta_{\mathcal A}=0.12\).
Here
\[
    \ln\frac{2N}{\delta}
    =
    \ln 80
    \approx
    4.382.
\]
The two terms on the right-hand side of
Eq.~(\ref{eq:app_theory_sample_complexity}) are
\[
\begin{aligned}
    \frac{32\bar s\ln(2N/\delta)}{\Delta_{\mathcal A}^2}
    &\approx175.3,\\
    \frac{16\ln(2N/\delta)}{3\Delta_{\mathcal A}}
    &\approx194.8.
\end{aligned}
\]
Thus the convenient sufficient condition in
Eq.~(\ref{eq:app_theory_sample_complexity}) would require
\(m=195\).  It is conservative because it separately bounds each of the two
terms in \(\varepsilon_m\) by \(\Delta_{\mathcal A}/2\).  Directly solving
\(\varepsilon_m(0.1)<0.12\), namely
\[
\displaystyle
    2\sqrt{\frac{2(0.0180)}{m}\ln80}
    +\frac{8}{3m}\ln80
    <0.12,
\]
gives \(m>188.19\), so the smallest integer sample size satisfying the exact
inequality is \(m=189\).

\paragraph{Consistency requirement on the parameters.}
The illustrative values above are not jointly free.  A rule that flips a
proportion \(s_a(Z)\) of assignments changes the induced assignment on at
most that fraction of variables, so its accuracy exceeds the always-retain
accuracy by at most the expected flip rate:
\(\operatorname{Acc}(h_a)\le(1-p)+\mathrm E[s_a]\le(1-p)+\bar s\).
Consequently,
\[
    \Delta_{\mathcal A}
    =
    A_{\mathcal A}^*-A_{\mathrm{sol}}^*
    \le
    b_G-p+\bar s,
\]
and the choice \(\bar s=0.0180\), \(\Delta_{\mathcal A}=0.12\) implicitly
requires \(b_G-p\ge0.102\): the Bayes error of static local prediction must
exceed the early-solution error rate by roughly ten percentage points or
more.  On WA, where the \(1.80\%\) median early-to-final difference
suggests \(p\approx0.018\), this amounts to \(b_G\gtrsim0.12\); that is,
even the best predictor using only the static local input must mislabel at
least about \(12\%\) of the binary variables.  This is a substantive assumption
about the difficulty of static prediction on the benchmark, not a
consequence of the theorems, and we state it so that the example is read as
a self-consistent operating point rather than a measured one.  Smaller,
still positive class advantages remain valid and simply rescale the
sufficient sample size, which grows like \(\Delta_{\mathcal A}^{-2}\) in
Eq.~(\ref{eq:app_theory_sample_complexity}).

Each MILP instance is one independent observation, so labels within an
instance may be arbitrarily dependent.  Accordingly, the penalty depends on
\(m\), the number of instances, rather than the total number of variable
labels.  The result assumes a fixed finite rule class selected by empirical
variable-level \(0\)--\(1\) prediction loss.  Its formal scope is finite-rule
model selection, providing an abstraction that isolates instance-level
estimation and the sparse-correction variance effect.  In practice the
predictor is a GNN trained by stochastic optimization with validation-based
checkpoint selection, a procedure better described as comparing a small,
data-dependent set of candidates---the checkpoints and configurations
actually evaluated---than as exact empirical risk minimization over a class
fixed in advance.  Theorem~\ref{thm:finite-sample} should therefore be read
as an idealized account of that selection stage: it isolates the
instance-level estimation cost and explains why sparse correction classes
are cheap to select among, but it is not a uniform-convergence guarantee
for the full GNN function class.  Statements in the main text about the
gain persisting under finite-sample model selection refer to this
abstraction layer.

\paragraph{Relation to the deployed algorithm.}
Two design choices of the deployed system fall outside the formal
statements above and are worth flagging explicitly.  First, the analysis
scores rules by variable-level \(0\)--\(1\) accuracy---equivalently,
thresholding the posterior at \(1/2\) over \emph{all} binary
variables---whereas the deployed method ranks variables by the predicted
consistency probability and fixes only the top-ranked ones under a budget
chosen by the downstream method.  The theorems thus quantify the
information available in the scores; how the scores are consumed by a
ranking-based fixing rule is evaluated empirically.  Second, the posteriors
in this appendix condition on a single early solution, whereas at
inference time the system averages, per variable, the logits computed from
the last \(K\) improving solutions after aligning them to the reference
solution \(X^{ES}\).  Because conditioning on the tuple of the last \(K\)
solutions only enlarges the input, Theorem~\ref{thm:info-gain} applies
verbatim to the ensembled variant, while the specific
alignment-and-averaging estimator is a practical variance-reduction
heuristic whose effect is evaluated empirically.

\end{document}